\documentclass{article}

\usepackage{microtype}
\usepackage{graphicx}
\usepackage{subfigure}
\usepackage{booktabs} % for professional tables

\usepackage{hyperref}

\usepackage{amsmath, amsthm, amssymb, amsfonts}
\usepackage{bm,bbm}
\usepackage{adjustbox}
\usepackage{tablefootnote}

\mathchardef\mhyphen="2D

\newcommand{\mc}[1]{\mathcal{#1}}
\newcommand{\mb}[1]{\mathbb{#1}}
\newcommand{\mr}[1]{\mathrm{#1}}
\newcommand{\norm}[1]{\left\|{#1}\right\|}
\DeclareMathOperator*{\argmin}{\arg\!\min}

\newtheorem{assumption}{Assumption}
\newtheorem{theorem}{Theorem}
\newtheorem{lemma}{Lemma}

\newtheorem{definition}{Definition}

\newtheorem{fact}{Fact}

\usepackage[accepted]{icml2020}

\icmltitlerunning{(Locally) DP Combinatorial Semi-Bandits}

\begin{document}

\twocolumn[
\icmltitle{(Locally) Differentially Private Combinatorial Semi-Bandits}



\icmlsetsymbol{equal}{*}
%\icmlsetsymbol{intern}{#}

\begin{icmlauthorlist}
\icmlauthor{Xiaoyu Chen}{equal,aff1}
\icmlauthor{Kai Zheng}{equal,aff1,intern}
\icmlauthor{Zixin Zhou}{aff3}
\icmlauthor{Yunchang Yang}{aff2}
\icmlauthor{Wei Chen}{aff4}
\icmlauthor{Liwei Wang}{aff1,aff2}
% \icmlauthor{Iaesut Saoeu}{ed}
% \icmlauthor{Fiuea Rrrr}{to}
% \icmlauthor{Tateu H.~Yasehe}{ed,to,goo}
% \icmlauthor{Aaoeu Iasoh}{goo}
% \icmlauthor{Buiui Eueu}{ed}
% \icmlauthor{Aeuia Zzzz}{ed}
% \icmlauthor{Bieea C.~Yyyy}{to,goo}
% \icmlauthor{Teoau Xxxx}{ed}
% \icmlauthor{Eee Pppp}{ed}
\end{icmlauthorlist}

\icmlaffiliation{aff1}{Key Laboratory of Machine Perception, MOE, School of EECS, Peking University}
\icmlaffiliation{intern}{Work done while interned at Microsoft Research Asia}
\icmlaffiliation{aff2}{Center for Data Science, Peking University}
\icmlaffiliation{aff3}{School of Electronics Engineering and Computer Science, Peking University}
\icmlaffiliation{aff4}{Microsoft Research Asia, Beijing, China}

\icmlcorrespondingauthor{Xiaoyu Chen}{cxy30@pku.edu.cn}
\icmlcorrespondingauthor{Kai Zheng}{zhengk92@pku.edu.cn}

\icmlkeywords{combinatorial semi-bandits, local differential privacy, differential privacy}

\vskip 0.3in
]



\printAffiliationsAndNotice{\icmlEqualContribution} % otherwise use the standard text.

\begin{abstract}
In this paper, we study Combinatorial Semi-Bandits (CSB) that is an extension of classic Multi-Armed Bandits (MAB) under Differential Privacy (DP) and stronger Local Differential Privacy (LDP) setting. Since the server receives more information from users in CSB, it usually causes additional dependence on the dimension of data, which is a notorious side-effect for privacy preserving learning. However for CSB under two common smoothness assumptions \cite{kveton2015tight,chen2016combinatorial}, we show it is possible to remove this side-effect. In detail, for $B_{\infty}$-bounded smooth CSB under either $\varepsilon$-LDP or $\varepsilon$-DP, we prove the optimal regret bound is $\Theta(\frac{mB^2_{\infty}\ln T } {\Delta\varepsilon^2})$ or $\tilde{\Theta}(\frac{mB^2_{\infty}\ln T} { \Delta\varepsilon})$ respectively, where $T$ is time period,  $\Delta$ is the gap of rewards and $m$ is the number of base arms, by proposing novel algorithms and matching lower bounds. For $B_1$-bounded smooth CSB under $\varepsilon$-DP, we also prove the optimal regret bound is $\tilde{\Theta}(\frac{mKB^2_1\ln T} {\Delta\varepsilon})$ with both upper bound and lower bound, where $K$ is the maximum number of feedback in each round. All above results nearly match corresponding non-private optimal rates, which imply there is no additional price for (locally) differentially private CSB in above common settings.

% In this paper, we study (locally) differentially private Combinatorial Semi-Bandits (CSB). Compared with private Multi-Armed Bandits (MAB), since the server receives more information from the user, it usually leads to additional dependence over the dimension of feedback, which is a notorious problem in private learning. 
% Somewhat surprisingly, we show that it is possible to remove this side-effect caused by privacy protection and nearly match corresponding non-private best results. In detail, for general CSB with $B$-bounded smooth reward function in the sense of \cite{chen2016combinatorial}, we propose a novel algorithm that achieves regret bound $\tilde{\mc{O}}(mB^2\log T / \varepsilon)$ over $T$ rounds under $\varepsilon$-local differential privacy, where $m$ is the number of base arms. However, for Linear CSB, $B$ equals $K$, where $K$ is the maximum number of feedback in each round, and above bound has an additional $K$ compared with non-private optimal result. We then propose a different algorithm with nearly optimal regret bound $\tilde{\mc{O}}(mK\log T / \varepsilon)$ if one cares about $\varepsilon$-differential privacy rather than $\varepsilon$-local differential privacy. Besides, we also prove some lower bounds in each setting.
\end{abstract}

\section{Introduction}
Stochastic Multi-Armed Bandits (MAB) \cite{bubeck2012regret} is a fundamental problem in machine learning with wide applications in real world. In stochastic MAB, there is an unknown underlying distribution over $[0,1]^m$ for $m$ base arms and a learner (or called a server) interacts with the environment for $T$ rounds. At each round, the environment draws random rewards from the distribution for $m$ base arms. At the same time, the learner chooses one of $m$ base arms based on previously collected information, and receives the reward of chosen arm. The goal of the learner is to minimize the regret, measured as the difference between the reward of best fixed base arm and the learner's total reward in expectation. Multi-Armed Bandits has been used in recommendation systems, clinical trial, etc. However, many of these applications rely heavily on users' sensitive data, which raise great concerns about data privacy. For example, in recommendation systems, observations at each round represent some preferences of the user over the recommended item set, which is the personal information of user $t$ and should be protected. 

Since first proposed in 2006, Differential Privacy (DP) \cite{dwork2006calibrating} has become a gold-standard in privacy preserving machine learning \cite{dwork2014algorithmic}. We say an algorithm protects differential privacy if there is not much difference between outputs of this algorithm over two datasets with Hamming distance 1 (see Section \ref{sec:preliminary} for the rigorous definition in the streaming setting). For $\varepsilon$-differentially private stochastic Multi-Armed Bandits, there has already been extensive studies \cite{mishra2015nearly,tossou2016algorithms,sajed2019optimal}. Based on classic non-private optimal UCB algorithm \cite{auer2002finite}, as well as the tree-based aggregation technique to calculate private summation \cite{dwork2010differential}, both \citet{mishra2015nearly} and \citet{tossou2016algorithms} designed algorithms under DP guarantee but with sub-optimal guarantee \footnote{In fact, \cite{tossou2016algorithms} achieved a better utility bound but under a weaker privacy guarantee compared with common differential privacy in the streaming setting.}. Recently, \citet{sajed2019optimal} proposed a complex algorithm based on non-private Successive Elimination \cite{even2002pac} and sparse vector technique \cite{dwork2014algorithmic} to achieve the optimal $\mc{O}(\frac{m\ln T}{\varepsilon\Delta})$ regret bound, where $\Delta$ is the minimum gap of rewards, and it matches both the non-private lower bound \cite{lai1985asymptotically} and the differentially private lower bound \cite{shariff2018differentially} in common parameter regimes.

However, stochastic MAB is the simplest model for sequential decision making with uncertainty. There are many problems in real world that have a combinatorial nature among multiple arms and maybe even non-linear reward functions, such as online advertising, online shortest path, online social influence maximization, etc, which can be modeled via Combinatorial Semi-Bandits (CSB) \cite{chen2013combinatorial,chen2016combinatorial,lattimore2018bandit}. In CSB, the learner chooses a super arm which is a set of base arms instead of a single base arm in MAB, and then observes the outcomes of the chosen arms as the feedback, and receive a reward determined by the chosen arms' outcomes. 
The reward can be a non-linear function in terms of these observations. Since many applications modeled via CSB also have issues about privacy leakage, in this paper, we study how to design private algorithms for Combinatorial Semi-Bandits under two common assumptions about non-linear rewards: $B_\infty$-bounded smoothness and $B_1$-bounded smoothness (see section \ref{sec:preliminary} for definitions.), which contain social influence maximization and linear CSB as important examples respectively \cite{kveton2015tight,chen2016combinatorial,wang2017improving}. 

\textbf{Main Difficulty:} Compared with simple stochastic MAB, it is more difficult to design differentially private algorithms for CSB, due to its large action space and non-linear rewards. Though each super arm in CSB can be regarded as a base arm in stochastic MAB, a straightforward implementation of differentially private algorithms for stochastic MAB will lead to a dependence over the size of decision set for super arms, which can be exponentially large in terms of $m$. Besides above two differences, we receive observations of a set of base arms contained in the chosen super arm at each round, instead of a single base arm in MAB. Denote the maximum cardinality of a super arm as $K$, which means the sensitive data collected at each round is roughly in a $K$-dimensional $L_\infty$ ball.
% Given above differences, we aim to design nearly optimal private algorithm in terms of $K$ and $m$ for CSB in this

However, protecting differential privacy usually causes an additional dependence on the dimension of data for utility guarantee compared with corresponding non-private result, which is a notorious side-effect of DP, such as in differentially private empirical risk minimization (ERM) \cite{bassily2014private}, bandits linear optimization \cite{agarwal2017price}, online convex optimization and bandits convex optimization \cite{thakurta2013nearly}, etc. On one hand, in some cases such as differentially private ERM \cite{bassily2014private}, this additional dependence on the dimension is unavoidable. On the other hand, some researchers show it is possible to eliminate this side-effect if there are some extra structures, such as assumptions about restricted strong convexity, parameter set in $L_1$ norm, or generalized linear model with data bounded in $L_2$ norm, etc \cite{kifer2012private,smith2013differentially,jain2014near,talwar2015nearly}. In general, it is unclear whether it is possible to eliminate the side-effect about dimensional dependence brought by privacy protection, let alone that our CSB setting does not have any extra structure mentioned above.

Besides, compared with differential privacy that admits the server to collect users' true data, local differential privacy (LDP) is a much stronger notion of privacy, which requires protecting data privacy before collection. Thus LDP is more practical and user-friendly compared with DP \cite{cormode2018privacy}. Intuitively, learning under LDP guarantee is more difficult as what we collect is already noisy.
Moreover, eliminating the side-effect on the dimension is also more difficult under LDP guarantee even when we have some extra assumptions. For example, there are some negative results for locally differentially private sparse mean estimation \cite{duchi2016minimax}.     

\textbf{Our Contributions:} Given above discussions, it seems hard to obtain nearly optimal regret for CSB under DP and much stronger LDP guarantee. Somewhat surprisingly, without any additional structure assumption such as sparsity, we show that it is indeed possible to achieve nearly optimal regret bound, by designing private algorithms with theoretical upper bounds and proving corresponding lower bounds in each case. Our upper bounds (nearly) match both our private lower bounds and non-private lower bounds  (see Table \ref{tab:results} for an overview, where $\Delta$ is some gap defined in Section \ref{sec: infinity bounded CSB}, $\mc{O}(\cdot)$ represents the upper bound, $\Theta$ represents both the upper bound and lower bound, and for $\tilde{\mc{O}}, \tilde{\Theta}$, we hide the poly-logarithmic dependence such as $\ln T, \ln m$). The main contributions of this paper are summarized as the follows:

{\bf (1)} For $B_\infty$-bounded smooth CSB under $\varepsilon$-LDP and $\varepsilon$-DP, we propose novel algorithms with regret bounds $\mc{O}(\frac{m B^2_{\infty} \ln T}{\varepsilon^2\Delta})$ and $\tilde{\mc{O}}(\frac{m B^2_{\infty} \ln T}{\varepsilon\Delta})$ respectively, and prove nearly matching lower bounds; 

{\bf (2)} For $B_1$-bounded smooth CSB under $\varepsilon$-DP, we propose an algorithm with regret bound $\tilde{\mc{O}}(\frac{m K B^2_1 \ln T}{\varepsilon\Delta})$ and nearly matching lower bound. 

% For $B_1$-bounded smooth CSB under $\varepsilon$-LDP, our algorithm achieves a regret bound in order $\mc{O}(\frac{mK^2B^2_1\ln T}{\varepsilon^2\Delta})$ automatically. Compared with non-private best result or our differentailly private result, there is an additional $K$ factor. We conjecture this additional $K$ is unavoidable under $\varepsilon$-LDP, but cannot prove it yet.

In Section \ref{sec:preliminary}, we provide some backgrounds in Combinatorial Semi-Bandits and (Local) Differential Privacy. Then in Section \ref{sec: infinity bounded CSB} and Section \ref{sec: L1CSB}, we study both upper and lower bounds for (locally) differentially private $B_{\infty}$-bounded smooth and $B_1$-bounded smooth CSB respectively. Finally, we conclude our main results in Section \ref{sec: conclusion}.

\renewcommand{\arraystretch}{1.8}
\begin{table*}[t] 
  \begin{adjustbox}{max width = 1.3\textwidth}
  \begin{tabular}{ | c | c | c | c | }
    \hline
    Problem  & $\varepsilon$-LDP & $\varepsilon$-DP & Non-Private Result \\ 
    \hline
    $B_{\infty}$-Smooth CSB&  $\Theta(\frac{m B^2_{\infty} \ln T}{\varepsilon^2\Delta})$ & $\tilde{\Theta}(\frac{m B^2_{\infty}\ln T}{\varepsilon \Delta})$ & $\Theta(\frac{m B^2_{\infty}\ln T}{\Delta})$ \cite{chen2016combinatorial,wang2017improving} \\
    \hline
    $B_1$-Smooth CSB  &  $\mc{O}(\frac{mK^2B^2_1\ln T}{\varepsilon^2\Delta})$ & $\tilde{\Theta}(\frac{mKB^2_1\ln T}{\varepsilon \Delta})$ & $\Theta(\frac{mKB^2_1\ln T}{\Delta})$ \cite{kveton2015tight,wang2017improving} \\
    \hline
  \end{tabular} 
  \end{adjustbox}
  \vspace{0.2cm}
  \caption{Summary of Our Results for Private CSB. $\Theta$ represents matching upper bounds and lower bounds. $\mc{O}$ represents upper bounds.  Our lower bound in DP setting is actually in an additive form, see Theorem \ref{theorem: lower bound for DP}. Here, we write it in a multiplicative form for simplicity, which is natural in common parameter regimes.} 
  \label{tab:results}
\end{table*}

\subsection{Other Related Work}
Besides differentially private stochastic MAB, there are also some works considering adversarial MAB with DP guarantee \cite{thakurta2013nearly,tossou2017achieving,agarwal2017price}. Later, \citet{shariff2018differentially} study contextual linear bandits under a relaxed definition of DP called \textit{Joint Differential Privacy}. Compared with DP, bandits learning with LDP guarantee is paid less attention to. Only \citet{gajane2018corrupt} study stochastic MAB under LDP guarantee. Recently, \citet{basu2019differential} investigate relations about several variants of differential privacy in MAB setting, and prove some lower bounds. For non-private Combinatorial Semi-Bandits, there is an extension of study \cite{gyorgy2007line,chen2013combinatorial,chen2016combinatorial,kveton2015tight,combes2015combinatorial,wang2017improving,wang2018thompson}.

% Combinatorial Semi-Bandits (CSB) \cite{chen2016combinatorial,lattimore2018bandit} is an extension of classic Stochastic Multi-Armed Bandits (MAB) \cite{bubeck2012regret} with wide applications in real world. In CSB, there is an unknown underlying distribution over $m$ base arms and a learner (or called a server) interacts with the environment for $T$ rounds. At each round, the environment draws a random observation from the distribution over $m$ base arms. At the same time, the learner chooses a super arm which is a set of base arms in the decision set based on previously collected information, and receives observations over the chosen super arm as well as a corresponding reward. The goal of the learner is to minimize the regret, measured as the difference between the reward of best fixed super arm and the learner's total reward in expectation. Combinatorial Semi-Bandits has been used in online advertising, social influence maximization, online shortest path etc \cite{chen2016combinatorial,lattimore2018bandit}. However, many of these applications rely heavily on users' data, which raise great concerns about data privacy. For example, in online advertising, observations at each round represent some preferences of the user over the recommended item set, which is the personal information of user $t$ and should be protected. Since proposed in 2006, Differential privacy \cite{dwork2006calibrating} has become a gold-standard in privacy preserving machine learning \cite{dwork2014algorithmic}. Thus, it is important to design private algorithms for Combinatorial Semi-Bandits.  
 
\section{Preliminaries}
\label{sec:preliminary}
Now we detail the concrete setting studied in this paper.
\subsection{Combinatorial Semi-Bandits}
In a Combinatorial Semi-Bandits (CSB), there are $m$ base arms (denote $[m] = \{1,2,\dots, m\}$), and a predefined decision set $\mc{S}\subset 2^m$, each element of which is a subset of $[m]$ with at most $K$ base arms and is called a super arm or an action, i.e. $|S|\leqslant K$ for any $S \in \mc{S}$ and $|\cdot|$ represents the cardinality of a set. $\mc{D}$ is an underlying unknown distribution supported on $[0,1]^m$ with expectation $\bm{\mu}=(\mu_1,\dots,\mu_m)$. There are $T$ rounds in total. At each round, the player chooses a super arm $S_t \in \mc{S}$, and the environment draws a fresh random outcome $\bm{X}_t = (X_{t,1},\dots,X_{t,m})$ from $\mc{D}$ independently of any other variables. Then the player receives a reward $R_t = R(S_t,\bm{X}_t)$ and observes the feedback $\{(i,X_{t,i}) | i \in S_t)\}$. We assume the reward function $R(\cdot,\cdot)$ satisfies following assumptions, which are common in either real applications or previous literature \cite{chen2016combinatorial,wang2018thompson}, such as Linear CSB, social influence maximization.

\begin{assumption}
There exists a reward function $r_{\bm{\mu}}(S)$ such that $\mb{E}[R(S,\bm{X})]=r_{\bm{\mu}}(S)$ for any $S\in \mc{S}$, where the expectation is over the randomness of outcome $\bm{X}$ and $\bm{\mu}=\mb{E}[\bm{X}]$. 
\end{assumption}
Under above assumption, define $\mathrm{opt}_{\bm{\mu}} = \max_{S\in \mc{S}} r_{\bm{\mu}}(S)$ as the optimal reward if we know $\bm{\mu}$ in advance. 

\begin{assumption}[$B_p$-bounded smoothness]
There exists a constant $B_p$, such that for arbitrary super arm $S$, and two mean vectors $\bm{\mu}, \bm{\mu'}$, there is $|r_{\bm{\mu}}(S)-r_{\bm{\mu'}}(S)| \leqslant B_p \norm{\bm{\mu}_S-\bm{\mu}'_S}_p)$, where $\bm{\mu}_S$ represents the truncated vector of $\bm{\mu}$ on subset $S$. 
\end{assumption}

\begin{assumption}[Monotonicity]
For any $\bm{\mu}, \bm{\mu'}$ such that $\bm{\mu} \leqslant \bm{\mu'}$ (element-wise compare), we have $r_{\bm{\mu}}(S) \leqslant r_{\bm{\mu'}}(S)$.
\end{assumption}

Intuitively, Assumptions 2 and 3 are about the smoothness and monotonicity of expected reward function $r_{\bm{\mu}}(\cdot)$, which are critical to deal with non-linear rewards $r_{\bm{\mu}}(S)$.

In this paper, we mainly consider two norms: $L_\infty$ norm $\norm{\cdot}_\infty$ and $L_1$ norm $\norm{\cdot}_1$. 
Important examples that satisfy $B_{\infty}$-bounded smoothness include social influence maximization and Probabilistic maximum coverage bandit \cite{chen2013combinatorial}. For $B_1$-bounded smooth CSB, online shortest path and online maximum spanning tree are typical applications \cite{wang2018thompson}. Obviously, Linear combinatorial semi-bandits is $B_1$-bounded smooth. We regard $B_\infty$ and $B_1$ as constants in the whole paper. Apparently, $B_{1}$-bounded smoothness is a weaker assumption compared with $B_{\infty}$-bounded smoothness, and we have the following fact:

\begin{fact}
\label{fact: bounded smoothness}
Suppose a reward function is $B_{\infty}$-bounded smooth, then it is also $B_1$-bounded smooth with $B_1=B_{\infty}$. On the contrary, suppose a reward function is $B_1$-bounded smooth, then it is $B_{\infty}$-bounded smooth with $B_{\infty} = KB_1$. 
\end{fact}

For many combinatorial problems such as MAX-CUT, Minimum Weighted Set Cover etc, there are only efficient approximation algorithms. Therefore, it is natural to model them as a general approximation oracle defined as below:
\begin{definition}
For some $\alpha,\beta \leqslant 1$, $(\alpha,\beta)$-approximation oracle is an oracle that takes an expectation vector $\bm{\mu}$ as input, and outputs a super arm $S \in \mc{S}$, such that $\Pr[r_{\bm{\mu}}(S)\geqslant \alpha \cdot \mr{opt}_{\bm{\mu}}] \geqslant \beta$. Here 
$\alpha$ is the approximation ratio and 
$\beta$ is the success probability of the oracle.
\end{definition}
With approximation oracle, we should then consider corresponding approximation regret as we can only solve offline problem approximately:
\begin{definition}
$(\alpha,\beta)$-approximation regret of a CMAB algorithm $\mc{A}$ after $T$ rounds using an $(\alpha,\beta)$-approximation oracle under the expectation
vector $\bm{\mu}$ is defined as $\mr{Reg}_{\bm{\mu},\alpha,\beta}(T):= T\cdot \alpha\beta \cdot\mr{opt}_{\bm{\mu}} - \mb{E}\left[\sum_{t=1}^Tr_{\bm{\mu}}(S_t)\right]$.
\end{definition}

\subsection{(Local) Differential Privacy}
Now we give definitions of DP and LDP, as well as a basic building block. 
\begin{definition}[Differential Privacy \cite{dwork2006calibrating,jain2012differentially}]
Let $D = \langle x_1, x_2, \dots, x_T \rangle$ be a sequence of data with domain $\mc{X}^T$. 
Let $\mathcal{A}(D) = Y$, where $Y = \langle y_1, y_2, \dots, y_T \rangle \in \mathcal{Y}^T$ be $T$ outputs of the randomized algorithm $\mathcal{A}$ on input $D$. 
$\mathcal{A}$ is said to preserve $\varepsilon$-differential privacy, if for any two data sequences $D, D'$ that differ in at most one entry, and for any subset $U \subset \mathcal{Y}^T$, it holds that
\[
  \Pr(\mathcal{A}(D)\in U) \leq e^\varepsilon \cdot \Pr(\mathcal{A}(D')\in U).
\]
% \[
%   \Pr(\mathcal{A}(D)\in U) \leq \Pr(\mathcal{A}(D')\in U)e^\varepsilon + \delta.
% \]
% In particular, if $\mathcal{A}$ preserves $(\varepsilon,0)$-differential privacy, we say $\mathcal{A}$ is $\varepsilon$-differentially private.
\end{definition}

Compared with DP, Local Differential Privacy (LDP) is a stronger notion of privacy than DP, see  \citet{kasiviswanathan2011can,duchi2013localb}. Since LDP requires to encrypt each user's data to protect privacy before collection, there is no need to define corresponding streaming version. Here we adopt the LDP definition given in \cite{bassily2015local}.
\begin{definition}[LDP]
    A mechanism $\mc{A}: \mc{X} \rightarrow \mc{Y}$ is said to be $\varepsilon$-local differential private or $\varepsilon$-LDP, if for any $x, x' \in \mc{X}$, and any (measurable) subset $U \subset \mc{Y}$, there is 
    \begin{align*}
        \Pr(\mc{A}(x) \in U) \leqslant e^{\varepsilon} \cdot \Pr(\mc{A}(x') \in U).
    \end{align*}
    % \begin{align*}
    %     \Pr[\mc{A}(x) \in U] \leqslant e^{\varepsilon} \Pr[\mc{A}(x') \in U] + \delta
    % \end{align*}
\end{definition}

To protect $\varepsilon$-LDP, the most commonly used method is Laplacian mechanism. Suppose the output domain $\mc{Y}$ of an algorithm $\mc{A}$ is bounded by a $d$-dimensional L1 ball with radius $R$, Laplacian mechanism just injects a $d$-dimensional random noise to the true output $\mc{A}(x)$, and each entry of noise is sampled from $\mr{Lap}(R/\varepsilon)$ independently \footnote{$\mr{Lap}(b)$ represents The Laplace distribution centered at $0$ with scale $b$, and its \textit{p.d.f} is $\mr{Lap}(x|b)=\frac{1}{2b}\exp(-\frac{|x|}{b})$. The corresponding variance is $2b^2$.}. It is easy to prove the Laplacian mechanism guarantees $\varepsilon$-LDP \cite{dwork2014algorithmic}. 

\section{$B_\infty$-Bounded Smooth CSB with Privacy Guarantee}
\label{sec: infinity bounded CSB}
Since learning under LDP is much more difficult compared with DP, we mainly consider how to design an optimal algorithm for $B_\infty$-Bounded Smooth CSB under $\varepsilon$-LDP guarantee. As we can see, based on our observation for locally differentially private CSB, it is then easy to obtain results for differentially private CSB. 

As a warm-up, we show that a simple mechanism can achieve non-trivial regret with LDP guarantee, but the dependence on dimension $K$ is sub-optimal. Next, we design an improved version with optimal utility bound, and the matching lower bound is proved in Subsection \ref{subsection: lower bound for LDP}. 

\subsection{A Straightforward Algorithm with Sub-Optimal Guarantee}
Our private algorithm is based on previous non-private CSB algorithm, Combinatorial UCB (CUCB) \cite{chen2013combinatorial,chen2016combinatorial}. Though the reward function is non-linear in terms of super arm $S$ and we only have access to some approximation oracle, which make our setting more complicated compared with previous private stochastic MAB ~\cite{mishra2015nearly,tossou2016algorithms,sajed2019optimal}, 
we show that the most straightforward method described in Algorithm \ref{algorithm: LDP CUCB1} (denoted as $\sf CUCB\mhyphen LDP1$), i.e. using Laplacian mechanism with respect to each user's data before collection, is enough to guarantee LDP and corresponding regret. 
% Note this idea is also used in \cite{tossou2017achieving}, but for \textit{differentially private} adversarial MAB.

The key observation is that, the mean estimation of each base arm lies at the core of CUCB algorithm, and adding a Laplacian noise with respect to each observation causes additional variance to these estimations, which can be handled by relaxed upper confidence bounds. Injecting noise to the reward is used both in \citet{tossou2017achieving} and \citet{agarwal2017price} for differentially private adversarial MAB. The idea about relaxed UCB also appears before for differentially private stochastic MAB \cite{mishra2015nearly}, whereas we study more general locally differentially private CSB with non-linear reward and approximation oracle. Given the
Laplacian mechanism, the privacy guarantee of Algorithm \ref{algorithm: LDP CUCB1} is obvious:

\begin{algorithm}[tb]
   \caption{$\sf CUCB\mhyphen LDP1$}
   \label{algorithm: LDP CUCB1}
\begin{algorithmic}[1]
   \STATE {\bfseries Input:} Privacy budgets $\varepsilon, \delta$
   \STATE {\bfseries Initialize:} $ \forall i \in [m], T_{0,i}=0$, empirical mean $\tilde{\mu}_0(i) = 0$.
   \FOR{$t=1,2, \dots$}
   \STATE $\forall i, \bar{\mu}_{t-1}(i) = \min\{ \tilde{\mu}_{t-1}(i) + 4\sqrt{\frac{2K\ln T}{\varepsilon^2 T_{t-1,i}}}, 1\}$ \footnotemark
   \STATE Play $S_t = \mathrm{Oracle}(\bar{\mu}_{t-1})$ if $\bar{\mu}_{t-1} \geqslant 0$ else $\forall S \in \mc{S}$
   \STATE User generates outcome $X_{t,i}$ for $i \in S_t$, and sends $X_{t,i} + z_{t,i}$ to the server, where $z_{t,i}\sim \mr{Lap}(K/\varepsilon)$
   \STATE Server updates $T_{t,i} = T_{t-1,i}+1, \tilde{\mu}_{t,i} = \frac{T_{t-1,i}\tilde{\mu}_{t-1,i}+X_{t,i} + z_{t,i}}{T_{t,i}}$, for $i \in S_t$, and keep others unchanged.
   \ENDFOR
\end{algorithmic}
\end{algorithm}
\footnotetext{If a denominator is $0$, we define corresponding constant as $+\infty$.} 

\begin{theorem}
\label{theorem: LDP CMAB}
Algorithm \ref{algorithm: LDP CUCB1} guarantees $\varepsilon$-LDP.
\end{theorem}
% \begin{proof}
% Because the information we collected at round $t$ is in a $|S_t|$ dimensional $L_1$ ball with radius at most $k$, the result holds trivially according to Laplace Mechanism.
% \end{proof}

Before stating the regret bound, we define some necessary notations. We say a super arm $S$ is bad if $r_{\bm{\mu}}(S) < \alpha \cdot \mr{opt}_{\bm{\mu}}$, and denote the set of bad super arms as $\mc{S}_B:=\{S\in \mc{S}|r_{\bm{\mu}}(S) < \alpha \cdot \mr{opt}_{\bm{\mu}}\}$. For any base arm $i\in [m]$, define
\begin{align}
    \Delta_{\min}^i := \alpha \cdot \mr{opt}_{\bm{\mu}} - \max\{r_{\bm{\mu}}(S)|S\in\mc{S}_B, i\in S\}, \\
    \Delta_{\max}^i := \alpha \cdot \mr{opt}_{\bm{\mu}} - \min\{r_{\bm{\mu}}(S)|S\in\mc{S}_B, i\in S\},
\end{align}
and $\Delta:=\min_{i\in[m]} \Delta_{\min}^i$.
% and $\Delta_{\max}:=\max_{i\in[m]} \Delta_{\max}^i, \Delta_{\min}:=\max_{i\in[m]} \Delta_{\min}^i$.

Now, we state the utility guarantee of Algorithm \ref{algorithm: LDP CUCB1}:

\begin{theorem}
\label{theorem: utility LDP CMAB}
Under $B_{\infty}$-bounded smoothness and monotonicity assumptions, the regret of Algorithm \ref{algorithm: LDP CUCB1} is upper bounded by
\begin{equation}
	Reg_{\mu,\alpha,\beta}(T) \leqslant \mc{O} \left( \sum_{i \in [m], \Delta_{\min}^i > 0} \frac{K^2 B_\infty^2 \ln T}{\varepsilon^2 \Delta_{\min}^i}\right).
\end{equation}
\end{theorem}
% \begin{proof}
% 	Nearly the same proof as \cite{chen2013combinatorial} except we set $\ell_t = \frac{48kL^2\ln(1.25/\delta)\ln t}{(\varepsilon \Delta_{\min})^2}$. The main idea is that: since we injects Laplace noise to each outcome at every round, according to concentration inequality, these noise will concentrate around 0, which only have influence over the convergence rate, i.e. we need more rounds to distinguish the gap between sub-optimal arms. Besides, since the oracle only admits positive input, we need to use a relatively larger upper confidence bound to correct noisy input.  
% \end{proof}

Compared with corresponding non-private CUCB that achieves $\mc{O} \left( \sum_{i \in [m], \Delta_{\min}^i > 0} \frac{B^2_{\infty}\ln T}{\Delta_{\min}^i} \right)$ regret \cite{chen2013combinatorial,chen2016combinatorial}, one can see the regret bound of Algorithm \ref{algorithm: LDP CUCB1} has an extra multiplicative factor $\frac{K^2}{\varepsilon^2}$, which is the price we pay for protecting LDP. According to our lower bound proved in Subsection \ref{subsection: lower bound for LDP}, the dependence on the privacy parameter $\varepsilon$ is optimal. However the additional term $K^2$ brought by privacy protection is undesirable and will hurt final performance for large $K$. In the next subsection, we show how to eliminate this additional $K^2$ factor.

\subsection{An Improved Algorithm with the Best Guarantee}
% To alleviate the notorious effect of privacy protection over dimension in utility guarantee, nearly all of previous study rely on some sparsity assumption, especially some key quantities are contained in a small $L_1$-norm ball \cite{jain2014near,talwar2015nearly,zheng2017collect}, thus it is possible to improve the utility guarantee in terms of dimension, as we only need to focus on some `compressed' information. These assumptions are intuitively reasonable in their settings.

Compared with the previous studies that try to eliminate the side-effect of dimension brought by privacy protection under either \textit{sparsity} or \textit{low complexity} assumptions \cite{jain2014near,talwar2015nearly,zheng2017collect}, in our general CSB setting, the information at each round is contained in a $K$-dimensional $L_\infty$ ball, and we do not have any sparsity assumption, which makes the additional $K^2$ factor seem unavoidable. 

Somewhat surprisingly, after a careful analysis, we find that there is some redundant information implicitly even without any sparsity assumption. 
In detail, in the analysis of Algorithm \ref{algorithm: LDP CUCB1}, the instant regret of choosing super arm $S_t$ at round $t$ is controlled by the largest mean estimation error among all base arms in $S_t$, 
which implies that we do not need to require all the observation of base arms in $S_t$ of user $t$ to update corresponding empirical means. Instead, we only use the observation of least pulled base arm in $S_t$ to update its empirical mean and keep others unchanged, as it is the weakest one in $S_t$ and causes largest estimation error. Since the user only sends the information of one entry to server now, it is enough to add noise in $\mc{O}(1/\varepsilon)$ order to protect it, which then gets rids of the annoying additional $K^2$ factor in the regret guarantee. Denote this variant as $\sf CUCB\mhyphen LDP2$, as shown in Algorithm \ref{algorithm: LDP CUCB2}. 

\begin{algorithm}[tb]
   \caption{$\sf CUCB\mhyphen LDP2$}
   \label{algorithm: LDP CUCB2}
\begin{algorithmic}[1]
   \STATE {\bfseries Input:} Privacy budgets $\varepsilon, \delta$
   \STATE {\bfseries Initialize:} $\forall i \in [m], T_{0,i}=0$, empirical mean $\tilde{\mu}_0(i) = 0$.
   \FOR{$t=1,2, \dots$}
   \STATE $\forall i, \bar{\mu}_{t-1}(i) = \min\{ \tilde{\mu}_{t-1}(i) + 4\sqrt{2\frac{\ln T}{\varepsilon^2 T_{t-1,i}}}, 1\}$ 
   \STATE Play $S_t = \mathrm{Oracle}(\bar{\mu}_{t-1})$ if $\bar{\mu}_{t-1} \geqslant 0$ else $\forall S \in \mc{S}$
   \STATE User generates outcome $X_{t,i}$ for $i \in S_t$, and sends $X_{t, I_t} + z_{t,I_t}$ to the server, where $I_t = \argmin_{i\in S_t} T_{t-1, i}, z_{t,I_t}\sim \rm{Lap}( 1/\varepsilon)$
   \STATE Server updates $T_{t,I_t} = T_{t-1,I_t}+1, \tilde{\mu}_{t,I_t} = \frac{T_{t-1,I_t}\tilde{\mu}_{t-1,I_t}+X_{t,I_t} + z_{t,I_t}}{T_{t,I_t}}$, and keep others unchanged.
   \ENDFOR
\end{algorithmic}
\end{algorithm}

Again, the privacy guarantee follows directly from the classic Laplacian mechanism:

\begin{theorem}
\label{theorem: LDP CUCB2}
Algorithm \ref{algorithm: LDP CUCB2} guarantees $\varepsilon$-LDP.
\end{theorem}
% \begin{proof}
% Because the information we collected at round $t$ is only $1$ dimensional, the result holds trivially according to Laplace Mechanism.
% \end{proof}

Since we condense the information required from each user significantly, which is reduced from $K$ observations to one observation, now we can inject less noise and prove a much better regret bound compared with the guarantee of Algorithm \ref{algorithm: LDP CUCB1}:

\begin{theorem}
\label{theorem: utility LDP CUCB2}
Under $B_{\infty}$-bounded smoothness and monotonicity assumptions, the regret of Algorithm \ref{algorithm: LDP CUCB2} is upper bounded by
\begin{equation}
	Reg_{\mu,\alpha,\beta}(T) \leqslant \mc{O} \left( \sum_{i \in [m], \Delta_{\min}^i > 0} \frac{B^2_{\infty}\ln T}{\varepsilon^2 \Delta_{\min}^i}\right)\\
\end{equation}
\end{theorem}

Compared with the non-private theoretical guarantee, theorem \ref{theorem: utility LDP CUCB2} implies that we can achieve optimal locally differentially private $B_\infty$-bounded smooth CSB without any additional price paid for privacy protection, which is a bit surprising given the previous work about (locally) differentially private learning. See section \ref{section: proof of upper bound of DLP} in the supplementary materials for the proof of theorem \ref{theorem: utility LDP CUCB2}.

% ($e_i$ is the $i$-th basis vector)
Multi-Armed Bandits (MAB) is a special case of CSB, where $\mc{S} = \{e_i | i\in [m]\}$ and $K=1$. In this case, our Algorithms \ref{algorithm: LDP CUCB1} and \ref{algorithm: LDP CUCB2}) are exactly the same, and we obtain an algorithm for MAB under $\varepsilon$-LDP with regret bound $\mc{O}(\sum_{i\neq i^*}\frac{\ln T}{\Delta_i\varepsilon^2})$, where $i^*$ is the optimal base arm, and $\Delta_i$ is the gap between arm $i$ and optimal arm $i^*$. Apparently, this regret bound is also optimal given the LDP lower bound $\Omega(\sum_{i \neq i^*} \frac{\ln T}{\varepsilon^2\Delta_i})$ proved in \citet{basu2019differential} and non-private lower bound $\Omega(\sum_{i \neq i^*} \frac{\ln T}{\Delta_i})$ \cite{bubeck2012regret}.

Finally, if one wants to protect $\varepsilon$-DP rather than $\varepsilon$-LDP, based on the same observation as above, we can simply use the tree-based aggregation technique \cite{dwork2010differential} with respect to the least pulled base arm to calculate its empirical mean estimation with DP guarantee. Since the tree-based aggregation technique injects much less noise compared with Algorithm \ref{algorithm: LDP CUCB2} designed for LDP, it is not hard to prove that this variant for DP can achieve regret bound $\tilde{\mc{O}}(\frac{m B^2_{\infty}\ln T}{\varepsilon \Delta})$.\footnote{The proof for this result is actually a combination of techniques used in this subsection and what we will use in subsection \ref{subsec: DP for L1}, hence omitted. }

\subsection{Lower Bounds}
\label{subsection: lower bound for LDP}

In this subsection, we prove the regret lower bound for locally private CSB problem with $B_{\infty}$-bounded smoothness. Like previous work \cite{kveton2015tight,wang2017improving}, we only consider lower bound with exact oracle, i.e. $\alpha = \beta = 1$. 
% As there is no lower bound for $B_\infty$-bounded smooth CSB even in non-private case, we prove a lower bound for linear CSB under $\varepsilon$-LDP guarantee, and the main focus here is the dependence over privacy parameter $\varepsilon$, which shows our upper bound proved for Algorithm \ref{algorithm: LDP CUCB2} is optimal. 

First we define a class of algorithms that we are interested in:

\begin{definition}
An algorithm is called consistent if for any suboptimal super arm $S$, the number of times $S$ is chosen by the algorithm is subpolynomial in $T$ for any stochastic CSB instance, i.e. $\mathbb{E}\left[N_S(T)\right] \leq o(T^{p})$ for any $0<p<1$.
\end{definition}

Our lower bound is derived for the consistent algorithm class, which is natural for the stochastic CSB and has been used for lower bound analysis in many previous results~\cite{lattimore2018bandit, basu2019differential, lai1985asymptotically, kveton2015tight}.

Our analysis focuses on CSB instances where the sub-optimality gap $\Delta$ of any super arms are equal. Since general CSB problem is harder than CSB problem with equal sub-optimality gap (The latter problem can be reduced to the former), our lower bound can be directly applied to general CSB class, with $\Delta$ replaced with $\Delta_{\operatorname{min}}^i$ for each base arm $i$.

\begin{theorem}
\label{theorem: lower bound for LDP, L_infty setting}
For any $m$ and $K$, and any $\Delta$ satisfying $0 < \Delta/B_{\infty} < 0.35$, the regret of any consistent $\varepsilon$-locally private algorithm $\pi$ on the CSB problem with $B_{\infty}$-bounded smoothness is bounded from below as  
$$
\liminf _{T \rightarrow \infty} \frac{Reg(T)}{\log T} \geq \frac{B^2_{\infty}(m-1)}{64 (e^{\varepsilon}-1)^2 \Delta}
$$
Specifically, for $0 < \varepsilon \leq 1/2$, the regret is at least
$$
\liminf _{T \rightarrow \infty} \frac{Reg(T)}{\log T} \geq \frac{B^2_{\infty}(m-1)}{128 \varepsilon^2 \Delta}
$$
\end{theorem}

The lower bound shows that Algorithm \ref{algorithm: LDP CUCB2} achieves optimal regret with respect to all the parameters of the CSB instance. The proof of the theorem is an almost direct reduction from private MAB. Previous result (Theorem 2 in \citet{basu2019differential}
) shows that the regret for any consistent $\varepsilon$-locally private algorithm for MAB is at least $\Omega\left(\frac{m \ln T}{\varepsilon^2\Delta}\right)$. Since any MAB instance is a special case of CSB with $B_{\infty} = 1$, the regret lower bounds for stochastic CSB with $B_{\infty} = 1$ follows directly by reduction. For general CSB problem with $B_{\infty}$-bounded smoothness, we consider a similar instance with the reward of each arm in MAB instance multiplied by $B_{\infty}$. See Section \ref{section: proof of lower bound for LDP, L_infty} in the supplementary materials for the detailed analysis. For $B_{\infty}$-bounded smooth CSB under DP setting, using nearly the same technique, it is not hard to prove that the corresponding lower bound is $\Omega(\frac{mB^2_{\infty}\ln T}{\varepsilon\Delta})$.

\section{$B_1$-Bounded Smooth CSB with Privacy Guarantee}
\label{sec: L1CSB}

\subsection{$B_1$-Bounded Smooth CSB under LDP}
\label{subsec: LDP-B_1}

Though our proposed Algorithm \ref{algorithm: LDP CUCB2} is already optimal for $B_\infty$-bounded smooth CSB, if we use it for $B_1$-bounded smooth CSB such as important linear CSB to protect $\varepsilon$-LDP, we will obtain its regret bound in order $\mc{O} \left( \sum_{i \in [m], \Delta_{\min}^i > 0} \frac{K^2 B_1^2 \ln T}{\varepsilon^2\Delta_{\min}^i} \right)$ due to Fact \ref{fact: bounded smoothness}. However, the optimal non-private regret bound for $B_1$-bounded smooth CSB is $\Theta \left( \sum_{i \in [m], \Delta_{\min}^i > 0} \frac{K B_1^2\ln T}{\Delta_{\min}^i} \cdot \right)$ \cite{kveton2015tight,wang2017improving}, which implies a gap with our locally differentially private upper bound. Is it possible to eliminate this additional $K$ just like in the previous locally differentially private $B_\infty$-bounded smooth CSB? First we prove a lower bound for $B_1$-Bounded Smooth CSB under LDP guarantee. Our result under $B_1$-bounded smoothness assumption can be applied to linear CSB problem by setting $B_1=1$.

\begin{theorem}
\label{theorem: lower bound for LDP}
For any $m$ and $K$ such that $m/K$ is an integer, and any $\Delta$ satisfying $0 < \Delta/(B_1K) < 0.35$, the regret of any consistent $\varepsilon$-locally private algorithm $\pi$ on the CSB problem satisfying $B_1$-bounded smoothness is bounded from below as  
$$
\liminf _{T \rightarrow \infty} \frac{Reg(T)}{\log T} \geq \frac{B^2_1(m-K) K}{64 (e^{\varepsilon}-1)^2 \Delta}
$$
Specifically, for $0 < \varepsilon \leq 1/2$, the regret is at least
$$
\liminf _{T \rightarrow \infty} \frac{Reg(T)}{\log T} \geq \frac{B^2_1(m-K)K}{128 \varepsilon^2 \Delta}
$$
\end{theorem}

We borrow the hard instance from \citet{kveton2015tight} to prove the lower bound. Consider a $K$-path semi-bandit problem with $m$ base arms. The feasible super arms are $m/K$ paths, each containing base arm $(i-1)K+1, (i-1)K+2,...,iK$ for $i \in \{1,...,m/K\}$. The reward of pulling super arm $S$ is $B_1$ times the sum of the weight $\tilde{w}_i$ for $i \in S$. The weights $\tilde{w}_i$ of the different base arms in the same super arm are identical, while the weights in the different paths are i.i.d sampled. Denote the best super arm as $S^*$,  The weight of each base arm is a Bernoulli random variable with mean:

$$
\bar{w}(i)=\left\{\begin{array}{ll}{0.5} &
{ i \in S^* } \\ 
{0.5-\Delta / (B_1K)} & {\text { otherwise }}
\end{array}\right.
$$

We use the general canonical bandit model \cite{lattimore2018bandit} to prove above theorem. See Section \ref{section: proof of lower bound for LDP, L1} in the supplementary materials for the detailed proof.

Though we can only prove a lower bound of $\Omega(\frac{mKB_1^2\ln T}{\varepsilon^2\Delta})$ in the same order as corresponding non-private optimal guarantee, we conjecture our lower bound is loose and the right lower bound is $\Omega(\frac{mK^2B_1^2\ln T}{\varepsilon^2\Delta})$. In other words, maybe there is indeed some side-effect for utility guarantee about the dimension $K$ if we hope to protect LDP. 
%Intuitively, for $B_\infty$ bounded smooth CSB, we save a factor of $K$ because our algorithm does not require updating $O(K)$ arms in each round, and updating the least observed arm is enough and thus 
%	privacy protection can be carried out without the added $K$ factor.
Intuitively, for $B_1$ bounded smooth CSB, we may have to update all arms in a played super arm for the regret guarantee (instead of only one arm as we did for $B_\infty$ bounded smooth CSB), and this
	makes the privacy protection harder with an extra factor of $K$.

Since \textit{differential privacy} is a relatively weaker notion compared with LDP, there may be some hope to further improve the regret bound if we focus on the guarantee of DP. In next two subsections, we show it is indeed true, by designing an $\varepsilon$-differentially private algorithm with regret bound $\tilde{\mc{O}} \left( \sum_{i \in [m], \Delta_{\min}^i > 0} \frac{KB_1^2 \ln^2 T}{\Delta_{\min}^i} + \frac{mKB_1\ln^3 T}{\varepsilon} \cdot \right)$, and proving a nearly matching lower bound.

\subsection{Upper Bound under DP}
\label{subsec: DP for L1}
% Different with LDP, which requires protecting privacy during data collection, DP only requires data protection before outputting information learned from history data, therefore it is usually possible to obtain better utility guarantees compared with LDP. The main idea for DP CMAB is to use tree based aggregation technique \cite{thakurta2013nearly} to protect privacy, which can significantly reduce the noise injected to each base arm. In this subsection, we consider the regime where $k$ is large compared with inverse gap.

Compared with LDP, in which case the learning algorithm (or the server) can only receives noisy information, DP only has some restriction for the output of an algorithm, and the server has authority to collect true data. Thus, it is possible to inject much less noise under DP setting via an economic allocation of privacy budget $\varepsilon$.  

We use tree-based aggregation scheme~\cite{dwork2009complexity,chan2011private} to protect $\varepsilon$-DP in our algorithm, which is an effective method in releasing private continual statistics over a data stream and frequently used in previous work, such as stochastic MAB \cite{mishra2015nearly,tossou2016algorithms}, Online Convex Optimization \cite{thakurta2013nearly}. Consider a data stream $(X_1,X_2,...,X_T)$ where $X_i \in [0,1]$. In each step $t$, the algorithm receives data $X_t$, and needs to output the sum $\bar{X}_t = \sum_{i=1}^{t}X_i$, while insuring that the output sequence $(\bar{X}_1,\bar{X}_2,...,\bar{X}_T)$ are $\varepsilon$-differentially private. Tree-based mechanism solves this problem in an elegant way with a binary tree. Each leaf node denotes data $X_t$ received in step $t$. Each internal node calculates the sum of data in the leaf nodes rooted at it. 
Notice that one only needs access to $\lceil \log t \rceil$ nodes and sums up the values on them in order to calculate $\bar{X}_t$. 
Using the Laplacian mechanism, previous results have shown that adding i.i.d Lap($\|X\|_{1}\log T/\varepsilon$) to each node ensures $\varepsilon$-differential privacy for the scheme as stated in the following lemma:
\begin{lemma}[\citet{dwork2010differential,chan2011private}]
    \label{lemma: tree-based aggregation}
    Tree-based aggregation scheme with i.i.d $\mr{Lap}(\|X\|_{1}\log T/\varepsilon)$ noise added to each node  is $\varepsilon$-differentially private.
\end{lemma}

In our CSB setting, we store a vector $\bm{X_t}$ with support at most $K$ in the leaf nodes of step $t$. Each internal node calculates the sum of $\bm{X_t}$ in the leaf nodes rooted at it. For each node, we add i.i.d $\operatorname{Lap}(2K \log T/\varepsilon)$ noise to each dimension of the vector stored on the node to guarantee $\varepsilon$-DP (See Algorithm~\ref{algorithm: DP CUCB}). 
Based on Lemma \ref{lemma: tree-based aggregation}, we have

\begin{algorithm}[tb]
    \caption{$\sf CUCB\mhyphen DP$}
    \label{algorithm: DP CUCB}
\begin{algorithmic}[1]
   \STATE {\bfseries Input:} Privacy budgets $\varepsilon, \delta$.
   \STATE {\bfseries Initialize:} $\forall i \in [m], T_{0,i}=0$, empirical mean $\tilde{\mu}_0(i) = 0$.
   \FOR{$t=1$ {\bfseries to} $T$}
   \STATE $\forall i, \bar{\mu}_{t-1}(i) = \min\{ \tilde{\mu}_{t-1}(i) + \sqrt{\frac{4\ln (m T)}{T_{t-1,i}}} + \frac{12K \ln^3{T}}{T_{t-1,i} \varepsilon}, 1\}$ 
   \STATE Play $S_t = \mathrm{Oracle}(\bar{\mu}_{t-1})$ if $\bar{\mu}_{t-1} \geqslant 0$ else $\forall S \in \mc{S}$
   \STATE User generates outcome $X_{t,i}$ for $i \in S_t$, and sends $X_{t,i}$ to the server
   \STATE Server updates base arms in $S_t$: $\tilde{\mu}_{t,i} = \frac{TreeBasedAggregation(\{X_{\tau,i}|\tau \in [t], i\in S_\tau\})}{T_{t,i}}$, $T_{t,i} = T_{t-1,i}+1$, and keeps others unchanged 
%   \STATE Server updates base arms in $S_t$ with minimum counter $T_{t-1,i}$: $\tilde{\mu}_{t,i} = \frac{TreeBasedAggregation(\{X_{\tau,i}|\tau \in [t], i\in S_\tau\})}{T_{t,i}}$, $T_{t,i} = T_{t-1,i}+1$ 
%   \STATE for other base arms in arm set $[m]$, $T_{t,i} = T_{t-1,i}$ 
   \ENDFOR
\end{algorithmic}
\end{algorithm}

\begin{theorem}
\label{theorem: Privacy for DP CUCB}
Algorithm \ref{algorithm: DP CUCB} guarantees $\varepsilon$-DP.
\end{theorem}
% \begin{proof}
% In each step $t$, user generates outcome $\bm{X}_t$ with support at most $K$. Since $X_{t,i}$ is $[0,1]$ bounded, the $L_1$ distance between two different outcome is at most $2K$. For any two neighboring data stream differing in at most one entry, at most $\log T$ nodes in the tree are influenced and each differs in at most $2K$ in $L_1$ distance. It is not hard to see the $L_1$ sensitivity of the tree is at most $2K \log T$. The result follows by Laplace mechanism. 
% \end{proof}

In Algorithm \ref{algorithm: DP CUCB}, when we need to estimate the mean weight $\mu_i$ based on the previous outcome $X_{t,i}$, we add additional Laplace noise to the sum of $X_{t,i}$ due to tree-based aggregation scheme. Note that the number of Laplace noises added (the number of nodes we access to) is only logarithmic. This means that the additional confidence bound due to Laplace noise is only $\tilde{\Theta}(1/T_{t-1,i})$ for base arm $i$ when it is pulled for $T_{t-1,i}$ times. Compared with the original bound for the sub-Gaussian noise which is of order $\tilde{\Theta}(\sqrt{1/T_{t-1,i}})$, the additional bound for Laplace noise enjoys better dependence on $T_{t-1,i}$. This helps us to separate the term of $\Delta$ and $\varepsilon$ in the regret via delicate analysis, and finally derive a nearly optimal bound in the additive form.

\begin{theorem}
\label{theorem: utility for DP CMAB}
Under $B_{1}$-bounded smoothness and monotonicity assumptions, the regret of Algorithm \ref{algorithm: DP CUCB} is upper bounded by\begin{align*}
 Reg_{\mu,\alpha,\beta}(T) &\leqslant \mc{O}\left(  \sum_{i \in[m],\Delta_{\min}^i>0}\frac{K B_1^2\ln^2 T }{\Delta^i_{\operatorname{min}}}\right) \\
& + \mc{O}\left(\frac{mK B_1 \ln^3 T \ln\left(\frac{B_1K \ln T}{\Delta_{\operatorname{max}}\varepsilon}\right)}{\varepsilon}\right) .
\end{align*}
\end{theorem}

Note when privacy parameter $\varepsilon$ is regarded as a constant which is common in real applications, the second term in the right hand side of above inequality is nearly dominated by the first term, which is almost the optimal regret bound in non-private $B_1$-bounded smooth CSB. Thus by relaxing LDP to DP, we have shown that it is possible to eliminate the side-effect on dimension induced by privacy protection and nearly match corresponding non-private optimal bound $\mc{O}(\sum_{i\in[m],\Delta_{\min}^i>0} \frac{KB_1^2\ln T}{\Delta_{\min}^i})$.

Before proving Theorem~\ref{theorem: utility for DP CMAB}, we present the following lemma. This lemma gives an upper bounds on the sub-optimal gap in round $t$, which helps to treat the $\Delta$ term and $\varepsilon$ term separately. We refer readers to Section~\ref{section: proof of upper bound for DP} of the supplementary materials for the proof of Lemma~\ref{lemma: regret decomposition for DP}.  

\begin{lemma}
\label{lemma: regret decomposition for DP}
Suppose $\Delta_{S_t}=\alpha r_{\mu}(S_{\mu}^*)- r_{\mu}(S_t)$. Denote $$F_t = \left\{ \Delta_{S_t} \leq B_1\sum_{i \in S_t} \left(4\sqrt{\frac{\ln T}{T_{t-1,i}}} + \frac{ 24K\ln^{3} T}{T_{t-1,i} \varepsilon} \right) \right\}.$$ 
Then the regret for Algorithm 3 is bounded by 
\begin{align}
    \label{inequality: regret decomposition for DP}
    Reg_{\mu,\alpha,\beta}(T) \leq \sum_{t\in[T]}\Delta_{S_t} \mathbf{1}\{F_t\} + 3 \sum_{i\in[m]} \Delta_{\operatorname{max}}^{i}
\end{align}
\end{lemma}

Now we are ready to prove Theorem \ref{theorem: utility for DP CMAB}.
\begin{proof} (proof of Theorem \ref{theorem: utility for DP CMAB})
We mainly analyze the first term of the RHS in Inq.~\ref{inequality: regret decomposition for DP}. Define $\hat{R}_T = \sum_{t\in[T]}\Delta_{S_t} \mathbf{1}\{F_t\}$. 
In step $t$, we consider the case that $F_t$ happens. Define $\bar{\Delta}_{S_t} = \max_{i \in S_t} \Delta_{\operatorname{min}}^i$. Since $\Delta_{\operatorname{min}}^i \leq \Delta_{S_t}, \forall i \in S_t$, we have $\bar{\Delta}_{S_t} \leq \Delta_{S_t}$. Then we have
\begin{align*}
    \Delta_{S_t}+\bar{\Delta}_{S_t} &\leq 2\Delta_{S_t} \\
    &\leq 2 B_1 \sum_{i \in S_t} \left(4\sqrt{\frac{\ln T}{T_{t-1,i}}} + \frac{ 24K\ln^{3}T}{T_{t-1,i} \varepsilon} \right)
\end{align*}
$$$$

that is,
\begin{align*}
    \Delta_{S_t} &\leq 2 \Delta_{S_t} - \bar{\Delta}_{S_t} \\
    &\leq 2 B_1 \sum_{i \in S_t} \left(4\sqrt{\frac{\ln  T}{T_{t-1,i}}} + \frac{ 24K\ln^{3} T}{T_{t-1,i} \varepsilon} - \frac{\bar{\Delta}_{S_t}}{2B_1K} \right) \\
    & \leq 2 B_1\sum_{i \in S_t} \left(4\sqrt{\frac{\ln T}{T_{t-1,i}}} + \frac{ 24K\ln^{3} T}{T_{t-1,i} \varepsilon} - \frac{\Delta_{\operatorname{min}}^i}{2B_1K} \right)
\end{align*}

Let $n^{i}_{\operatorname{max}} = \max \left\{ \frac{256B^2_1 K^2\ln{T}}{(\Delta^i_{\operatorname{min}})^2}, \frac{96B_1K^2\ln^3 T}{\varepsilon\Delta^i_{\operatorname{min}}} \right\}$. Denote $\Delta_i(n) = 4B_1\sqrt{\frac{\ln T}{n}} + \frac{ 24B_1K\ln^{3} T}{n \varepsilon} - \frac{\Delta_{\operatorname{min}}^i}{2K}$. For base arm $i$, if $n \geq n^{i}_{\operatorname{max}}$, we have $\Delta_i(n) \leq 0$.

\begin{align*}
    \hat{R}_T  \leq & \sum_{t \in [T]} \Delta_{S_t} \mathbf{1}\{F_t\} \\
     \leq &\sum_{t=1}^{T}\sum_{i \in S_t} 2 \Delta_i(T_{i,t}) \mathbf{1}\{F_t\} 
\end{align*}
\begin{align*}
     \leq &\sum_{i\in[m]}\sum_{n=1}^{n^{i}_{\operatorname{max}}}2B_1 \left(4\sqrt{\frac{\ln  T}{n}} + \frac{ 24K\ln^{3} T}{n \varepsilon} \right) \\
     \leq &\sum_{i\in [m]} \int_{0}^{n^i_{\operatorname{max}}}8B_1 \sqrt{\frac{\ln  T}{n}} dn  \\
    &  +  \sum_{i\in [m]}\left(\frac{ 48B_1 K\ln^{3} T}{ \varepsilon} + \int_{1}^{n^i_{\operatorname{max}}}\frac{ 48B_1K\ln^{3} T}{n \varepsilon} dn \right) \\
    \leq & \sum_{i\in[m]} 16B_1 \sqrt{\ln T\left(\frac{256B^2_1 K^2\ln{ T}}{(\Delta^i_{\operatorname{min}})^2}+ \frac{96B_1K^2\ln^3 T}{\varepsilon\Delta^i_{\operatorname{min}}}\right)} +\\
    & \sum_{i\in [m]}\left(\frac{ 48B_1K\ln^{3} T}{ \varepsilon} \left(1 + \ln \left(\frac{256 K^2B^2_1\ln^4{T}}{\varepsilon(\Delta^i_{\operatorname{min}})^2} \right)\right)\right)
\end{align*}

After simplifying the equation using basic inequalities such as $\sqrt{ab} \geq \frac{2}{1/a+1/b}$ and $\sqrt{a+b} \leq \sqrt{a} + \sqrt{b}$ ($a,b \geq 0$), we can show that 
\begin{align*}
    &\hat{R}_T \leq  \\
    &\mc{O}\left( B^2_1 K \ln^2 T \sum_{i \in[m]}\frac{1}{\Delta^i_{\operatorname{min}}} + \frac{mB_1 K \ln^3 T \ln\left(\frac{B_1K \ln T}{\Delta_{\operatorname{max}}\varepsilon}\right)}{\varepsilon}\right)
\end{align*}

\end{proof}

\subsection{Lower Bound under DP}
\label{subsec: DP lower bound}
In this subsection, we prove the lower bound for CSB algorithm under $\varepsilon$-DP. 
%Our analysis focuses on CSB instances where the sub-optimality gap $\Delta$ of any super arms are equal. Since general CSB problem is harder than CSB problem with equal sub-optimality gap (The latter problem can reduce to the former problem), our lower bound can be applied directly to general CSB class, with $\Delta$ replaced with $\Delta_{\operatorname{min}}^i$ for each base arm $i$. 
Similar with the result of LDP lower bound, we consider CSB algorithm with consistent property. The lower bound stated below implies that our algorithm \ref{algorithm: DP CUCB} can achieve near-optimal regret regardless of logarithmic factors:

%\footnote{See Section \ref{subsection: lower bound for LDP} for the definition of consistent algorithm.}

\begin{theorem}
\label{theorem: lower bound for DP}
For any $m$ and $K$ such that $m \geq 2K$, and any $\Delta$ satisfying $0 < \Delta/(B_1K) < 0.35$, the regret for any consistent CSB algorithm guaranteeing $\varepsilon$-DP is at least $\Omega\left(\frac{B_1^2mK \ln T}{\Delta} + \frac{B_1m K\ln T}{\varepsilon}\right)$.
\end{theorem}

The theorem is proved in section \ref{section: proof of lower bound for DP} of the supplementary materials. We only sketch the proof here. Previous results have shown that for non-private stochastic linear CSB, the regret lower bound is at least $\Omega(\frac{mK \ln T}{\Delta})$. By slightly modifying the hard instance, we can show that the regret lower bound for non-private CSB with $B_1$-bounded smoothness is $\Omega(\frac{B_1^2mK \ln T}{\Delta})$. Since private CSB is strictly harder than non-private CSB (by reduction), the regret lower bound for private CSB is $\Omega(\frac{B_1mK \ln T}{\Delta})$. We only need to prove that the regret lower bound for private CSB is $\Omega\left(\frac{B_1m K\ln T}{\varepsilon}\right)$, from which we can prove that the regret lower bound is $\Omega\left( \max\left\{\frac{B_1^2mK \ln T}{\Delta}, \frac{B_1mK \ln T}{\varepsilon}\right\} \right) = \Omega\left(\frac{B_1^2mK \ln T}{\Delta} + \frac{B_1m K\ln T}{\varepsilon}\right)$.

Now we sketch the proof of $\Omega\left(\frac{B_1m K\ln T}{\varepsilon}\right)$ term. Note a simple extension of \citet{kveton2015tight} can only achieve $\Omega\left(\frac{B_1m \ln T}{\varepsilon}\right)$ in our differentially private setting, which is not satisfactory. It is thus necessary to construct some new hard instance to prove Theorem \ref{theorem: lower bound for DP}.

% \citet{kveton2015tight} propose the $K$-path semi-bandit problem as a special case of linear CSB to prove the non-private CSB lower bound. If we consider this problem in private CSB setting with $B_1$-bounded smoothness, the regret lower bound is only $\Omega\left(\frac{B_1m \ln T}{\varepsilon}\right)$. Besides, the reduction from MAB problem can only guarantee lower bound of $\Omega\left(\frac{B_1m \ln T}{\varepsilon}\right)$ \footnote{Similar as the proof of Theorem \ref{theorem: lower bound for LDP, L_infty setting}, we need to slightly modify the MAB problem to meet $B_1$-bounded smoothness assumption.}. Denote the number of times the sub-optimal super arm $S$ is pulled as $N_S$. With the technique of differential privacy, considering $K$-path semi-bandit problem can guarantee that  $\mathbb{E}[N_S]$ is at least $\Omega(\frac{mKB_1 \ln T}{\varepsilon \Delta})$ for any sub-optimal super arm $S$ with high probability. However, there are only $\Theta(m/K)$ sub-optimal super arms. On the other hand, there are $\Theta(m)$ sub-optimal super arms if we consider the reduction from MAB, but the lower bound of $\mathbb{E}[N_s]$ is $\Omega(\frac{mB_1 \ln T}{\varepsilon \Delta})$. As a result, Both these two methods fail to tackle the problem.

To solve this problem, we design the following CSB problem as a special case of general CSB with $B_1$-bounded smoothness. Suppose there are $m$ base arms, each associated with a weight sampled from Bernoulli distribution. These $m$ base arms are divided into three sets, $S^*, \tilde{S}$ and $\bar{S}$. $S^*$ contains $m$ base arms, which build up the optimal super arm set. $\tilde{S}$ contains $K-1$ ``public'' base arms for sub-optimal super arms. These arms are contained in all sub-optimal super arms. $\bar{S}$ contains $m-2K+1$ base arms. each base arm combined with $K-1$ "public" base arms in $\tilde{S}$ builds up a sub-optimal super arm. Totally we have $m-2K+1$ sub-optimal super arms and one optimal super arm. The mean of the Bernoulli random variable associated to each base arm is defined as follow:

$$
w(i)=\left\{\begin{array}{ll}{0.5} &
{ i \in S^* } \\ 
{0.5-\Delta / (B_1K)} & {\text { otherwise }}
\end{array}\right.
$$

The weights of base arms in $\tilde{S}$ are identical, while other weights are i.i.d sampled. The reward of pulling a super arm $S$ is $B_1$ times the sum of weights of all base arm $i \in S$. As a result, the sub-optimality gap of each sub-optimal super arm is $\Delta$. With the coupling argument in \citet{karwa2017finite}, we can prove that $\mathbb{E}(N_S)$ is at least $\Omega(\frac{mKB_1 \ln T}{\varepsilon \Delta})$ for any sub-optimal super arm $S$ with high probability. Since there are $\theta(m)$ sub-optimal super arm, we can reach the conclusion that the regret lower bound for private CSB is $\Omega\left(\frac{B_1m K\ln T}{\varepsilon}\right)$.

\section{Conclusion and Future work}
\label{sec: conclusion}

In this paper, we study (locally) differentially private algorithm for Combinatorial Semi-Bandits under two common assumptions about reward functions. For $B_{\infty}$-bounded smooth CSB under $\varepsilon$-LDP and $\varepsilon$-DP, we show the optimal regret of these two settings are respectively $\Theta(\frac{m B^2_{\infty} \ln T}{\varepsilon^2\Delta})$ and $\tilde{\Theta}(\frac{mB^2_{\infty} \ln T}{\varepsilon\Delta})$, by proving lower bounds and designing (nearly) optimal private algorithms. For relatively weaker $B_1$-bounded smooth CSB, if we are required to protect $\varepsilon$-DP instead of $\varepsilon$-LDP, we show the optimal regret is $\tilde{\Theta}(\frac{mKB^2_1 \ln T}{\varepsilon\Delta})$, and give a differentially private algorithm as well as a nearly matching lower bound. Moreover, above optimal performance in our (locally) differentially private CSB is nearly the same order as non-private setting~\cite{kveton2015tight,chen2016combinatorial,wang2017improving}. 

% propose an algorithm with regret $\tilde{\mc{O}}(\frac{mKB^2_1\ln^2 T}{\Delta } + \frac{mKB_1\ln^3 T}{ \varepsilon})$ for CSB with $B_1$-bounded smoothness. The regret bounds of our algorithm match the regret of CSB algorithms in  Besides, we also prove the regret lower bounds for both setting, which indicates our algorithms are near-optimal in (locally) differentially private setting.

Our Algorithm \ref{algorithm: LDP CUCB2} is applicable for locally private CSB with $B_1$-bounded smoothness, with a regret upper bound of $\mc{O}(\frac{mK^2B^2_1\ln T}{\Delta \varepsilon^2})$ in this setting. However, the regret lower bound we prove is just $\Omega(\frac{mKB^2_1\ln T}{\Delta \varepsilon^2})$. We conjecture that our lower bound is loose and the Algorithm \ref{algorithm: LDP CUCB2} is also near-optimal for locally private CSB with $B_1$-bounded smoothness. How to improve the lower bound is left as future work.

Recently, there are interesting results under Gini-weighted smoothness assumptions~\cite{merlis2019batch,merlis2020tight}. Compared with general Lipschitz smoothness considered in this work, this is a more refined smoothness assumption, which leads to near optimal regret bounds with less dependence on the dimension $K$. Directly applying our algorithms to this setting will lead to an additional dependence on $K$. How to remove this additional price for privacy preserving, and how to prove the corresponding lower bounds, are interesting problems for future work.
\section*{Acknowledgements}
We thank Siwei Wang for helpful discussions in the early stage of this work. This work was supported by National Key R\&D Program of China (2018YFB1402600), BJNSF (L172037), Key-Area Research and Development Program of Guangdong Province (No. 2019B121204008)] and Beijing Academy of Artificial Intelligence
%This work is supported by National Basic Research Program of China (973 Program) (grant no. 2015CB352502), NSFC (61573026), BJNSF (L172037) and Beijing Acedemy of Artificial Intelligence.

% \section*{Acknowledgements}

% \textbf{Do not} include acknowledgements in the initial version of
% the paper submitted for blind review.

% If a paper is accepted, the final camera-ready version can (and
% probably should) include acknowledgements. In this case, please
% place such acknowledgements in an unnumbered section at the
% end of the paper. Typically, this will include thanks to reviewers
% who gave useful comments, to colleagues who contributed to the ideas,
% and to funding agencies and corporate sponsors that provided financial


\bibliography{privateCMAB}
\bibliographystyle{icml2020}

\clearpage
\newpage

\onecolumn
\section*{Appendices}
\addcontentsline{toc}{section}{Appendices}
\renewcommand{\thesubsection}{\Alph{subsection}}

\subsection{Proof of Theorem \ref{theorem restate: utility LDP CUCB2}}

\label{section: proof of upper bound of DLP}

\setcounter{theorem}{3}
\begin{theorem}
\label{theorem restate: utility LDP CUCB2}
(Restate) For Algorithm 2, we have 
\begin{equation}
	Reg_{\mu,\alpha,\beta}(T) \leqslant \mc{O} \left( \sum_{i \in [m], \Delta_{\min}^i > 0} \frac{B^2_{\infty}\ln T}{\varepsilon^2 \Delta_{\min}^i} \right)\\
\end{equation}
\end{theorem}

\begin{proof}
Suppose $G_t$ denote the event that the oracle fails to produce an $\alpha$-approximate answer with respect to the input vector in step $t$. We have $\mathbb{P}[G_t] \leq 1-\beta$. The number of times $G_t$ happens in expectation is at most $(1-\beta)T$. The cumulative regret in these steps is at most $R_{\operatorname{fail}} \leq (1-\beta)T \Delta_{\operatorname{max}}$

Now we only consider the steps $G_t$ doesn't happen. We maintain counters $N_{i}$ in the proof, and denote its value in step $t$ as $N_{t,i}$. The initialization of $N_{t,i}$ is the same as $T_{t,i}$, i.e. $N_{0,i} = 0$. In step $t$, if $G_t$ doesn't happen, and the oracle selects a sub-optimal super arm, we increment $N_{I_t}$ by one, i.e. $N_{t,I_t} = N_{t-1,I_t}+1$, where $I_t = \argmin_{i\in S_t} T_{t-1, i}$, otherwise we keep $N_{i}$ unchanged. This indicates that $N_{t,i} \leq T_{t,i}$. Notice that if a sub-optimal super arm $S_t$ is pulled in step $t$, exactly one counter $N_{I_t}$ is incremented by one, and $I_t \in S_t$. As a result, we have:
\begin{align}
\label{equation: DP regret}
    Reg_{\mu,\alpha, \beta}(T) 
    \leq &T \alpha \beta \operatorname{opt}_{\mu} - \mathbb{E}\sum_{t=1}^T r_{\mu}(S_t)\notag\\
    \leq & R_{\operatorname{fail}} + T \alpha \beta \operatorname{opt}_{\mu} - 
    \left(T\alpha \operatorname{opt}_{\mu} -\sum_{i\in[m],\Delta_{\min}^i>0}\sum_{j=1}^{N_{T,i}}\Delta_{i,j}\right)\notag \\
    \leq &\sum_{i\in[m],\Delta_{\min}^i>0}\sum_{j=1}^{N_{T,i}}\Delta_{i,j}
\end{align}
Here $\Delta_{i,j}$ denote the suboptimal gap $\alpha \cdot \operatorname{opt}_{\mu}- r(S_t)$ when $N_{i}$ incremented from $j-1$ to $j$ in a certain step $t$.
%Since $N_{t,i}$ only updates in the steps that $F_t$ doesn't happen, the number of times base arm $i$ is pulled in the

Now we only need to bound $N_{T,i}$ and $\Delta_{i,j}$. We denote the following event as $\Lambda_{t,i}$: For a fixed step $t \in T$ and a fixed base arm $i \in [m]$,
$$\left|\tilde{\mu}_{t}(i)-\mu_i\right| \leq 4\sqrt{\frac{2\ln T}{\varepsilon^2 T_{t,i}}}.$$

The noise in $\tilde{\mu}_t(i)$ comes from two parts: the Laplacian noise added for privacy and the randomness of $X_{t,i}$. For the first part, by Bernstein's Inequality over $T_{t,i}$ i.i.d Laplace distribution, the confidence bound is $2\sqrt{\frac{2\ln T}{\varepsilon^2 T_{t,i}}}$ with prob. at least $1-2/T^2$. For the second part, since $X_{t,i}$ is $[0,1]$ bounded, the confidence bound is $2\sqrt{\frac{2\ln T}{T_{t,i}}}\leq 2\sqrt{\frac{2\ln T}{\varepsilon^2 T_{t,i}}}$ with prob. at least $1-2/T^2$ by Hoeffding's inequality. This shows that $\Lambda_{t,i}$ happens with prob. $1-4/T^2$. By union bounds over all steps, $\Lambda_{t,i}$ happens for all $t$ and $i$ with prob. $1-4/T$. We denote this event as $\Lambda$.

Suppose $\Lambda$ happens, we have $\mu(i) \leq \bar{\mu}_t(i) \leq \mu(i) + 4\sqrt{\frac{2\ln T}{\varepsilon^2 T_{t,i}}}$. If a sub-optimal arm $S_t$ is pulled in step $t$. we have 
\begin{align}
    \label{equation: comfidence bound for DP}
    \alpha r_{\mu}(S_{\mu}^*)- r_{\mu}(S_t) \leq & \alpha r_{\bar{\mu}_t}(S_{\mu}^*) - (r_{\bar{\mu}_t}(S_t) - B_{\infty}\|\bar{\mu}_{t}-\mu\|_{\infty}) \notag\\
    \leq & B_{\infty}\|\bar{\mu}_{t}-\mu\|_{\infty} \notag\\
    \leq & B_{\infty}(\|\bar{\mu}_{t}-\tilde{\mu}_t\|_{\infty}+\|\tilde{\mu}_{t}-\mu\|_{\infty}) \notag\\
    \leq & B_{\infty} 8\max_{i \in S_t}\left\{\sqrt{\frac{2\ln T}{\varepsilon^2 T_{t-1,i}}}\right\} \notag\\
    \leq & B_{\infty} 8\max_{i \in S_t}\left\{\sqrt{\frac{2\ln T}{\varepsilon^2 N_{t-1,i}}}\right\}
\end{align}
The first inequality is due to monotonicity and $B_{\infty}$-bounded smoothness assumption. The second inequality is because the oracle returns $S_t$ which satisfies $r_{\bar{\mu}_t}(S_t) \geq \alpha r_{\bar{\mu}_t}(S_{\mu}^*)$. The third inequality is due to the definition of $\bar{\mu}_t$ and the concentration bound for $\tilde{\mu}_t$. The last inequality is due to $N_{t,i} \leq T_{t,i}$.

Define $\bar{\Delta}_{S} = \max_{i \in S}\Delta_{\operatorname{min}}^i$. If $N_{t-1,i} > \frac{128B^2_{\infty} \ln T}{\varepsilon^2 \bar{\Delta}^2_{S_t}}$ for any $i \in S_t$, we have $\alpha r_{\mu}(S_{\mu}^*)- r_{\mu}(S_t) < \max_{i\in S_t} \Delta_{\operatorname{min}}^i$ by Equ. \ref{equation: comfidence bound for DP}. On the other hand, by the definition of $\Delta_{\operatorname{
min}}^i$, $\alpha r_{\mu}(S_{\mu}^*)- r_{\mu}(S_t) = \alpha \operatorname{opt}_{\mu}- r_{\mu}(S_t) \geq \max_{i \in S_t}\Delta_{\operatorname{min}}^i$, which leads to a contradiction. This means that if sub-optimal arm $S_t$ is pulled in step $t$, and $S_t$ contains base arm $i$, the counter $N_{t-1,i}$ is at most $\frac{128B^2_{\infty} \ln T}{\varepsilon^2 \bar{\Delta}^2_{S_t}} \leq \frac{128B^2_{\infty} \ln T}{\varepsilon^2 (\Delta^{i}_{\operatorname{min}})^2}$. That is, under high probability event $\Lambda$, the counter $N_{i}$ is at most $\frac{128B^2_{\infty} \ln T}{\varepsilon^2 \bar{\Delta}^2_{S_t}}$.

Besides, by Equ. \ref{equation: comfidence bound for DP}, we know that $\Delta_{i,j} \leq 8B_{\infty}\sqrt{\frac{2 \ln T}{\varepsilon^2 j-1}}$, since $N_{t-1,i}$ is the minimum counter in $\{N_{t-1,i}, i \in S_t\}$ and increments by one in step $t$.

Combining with Equ. \ref{equation: DP regret}, we have 

\begin{align*}
    Reg_{\mu,\alpha, \beta}(T) \leq & \sum_{i\in [m],\Delta_{\min}^i>0}\sum_{j=1}^{N_{T,i}}\Delta_{i,j} \\
    \leq &\sum_{i\in [m],\Delta_{\min}^i>0}\sum_{j=1}^{N_{T,i}} 8B_{\infty}\sqrt{\frac{2 \ln T}{\epsilon^2 j}} + 2m \Delta_{\max} \\
    \leq &\sum_{i\in [m],\Delta_{\min}^i>0} \int_{0}^{N_{T,i}}8B_{\infty}\sqrt{\frac{2 \ln T}{\epsilon^2 j}} dj + 2m \Delta_{\max}\\
    \leq & \sum_{i\in [m],\Delta_{\min}^i>0}  \frac{128B^2_{\infty}\ln T}{\varepsilon^2 \Delta_{\min}^i} + 2m \Delta_{\max}
\end{align*}

Considering $T$ as the dominant term, we reach the result.

\end{proof}

\subsection{Proof of Theorem \ref{theorem restate: lower bound for LDP, L_infty setting}}
\label{section: proof of lower bound for LDP, L_infty}
\begin{theorem}
\label{theorem restate: lower bound for LDP, L_infty setting}
For any $m$ and $K$, and any $\Delta$ satisfying $0 < \Delta/B_{\infty} < 0.35$, the regret of any consistent $\varepsilon$-locally private algorithm $\pi$ on the CSB problem with $B_{\infty}$-bounded smoothness is bounded from below as  
$$
\liminf _{T \rightarrow \infty} \frac{Reg(T)}{\log T} \geq \frac{B^2_{\infty}(m-1)}{64 (e^{\varepsilon}-1)^2 \Delta}
$$
Specifically, for $0 < \varepsilon \leq 1/2$, the regret is at least
$$
\liminf _{T \rightarrow \infty} \frac{Reg(T)}{\log T} \geq \frac{B^2_{\infty}(m-1)}{128 \varepsilon^2 \Delta}
$$
\end{theorem}

\begin{proof}
We slightly modify the MAB instance in \citet{basu2019differential}. Suppose there are $m$ arms in a MAB problem. Each arm $i \in [m]$ is associated with an i.i.d Bernoulli random variable $\mu$ with mean $\bar{\mu}_{i}$. If arm $i$ is pulled in a certain step $t$, instead of receiving reward $\tilde{\mu}(i)$ sampled from the distribution of $\mu$, we receive a reward of $B_{\infty} \cdot \tilde{\mu}(i)$. Denote the sub-optimality gap of pulling a sub-optimal arm as $\Delta$. Following the argument in \citet{basu2019differential}, we consider two "MAB" instance: $\nu_1$ with mean weight $\bar{\mu} = \{\Delta/B_{\infty},0,...,0\}$ and $\nu_2$ with  $\bar{\mu} = \{\Delta/B_{\infty},...,0, 2\Delta/B_{\infty}\}$. Similarly, we can show that each supoptimal arm need to be pulled at least $$\frac{1}{2\min\{4,e^{2\varepsilon}\}(e^{\varepsilon}-1)^2 D(f_a\|f^*)},$$
where $f_a$ and $f^*$ denote the weight distribution of arm $a$ and optimal arm. Since $D(f_a\|f^*) \leq 4 \Delta^2/B_{\infty}^2$, we have 
\begin{align*}
\liminf_{T \rightarrow \infty} \frac{Reg(T)}{\ln T} \geq& (m-1)\frac{1}{2\min\{4,e^{2\varepsilon}\}(e^{\varepsilon}-1)^2 D(f_a\|f^*)} \Delta\\
\geq& (m-1)\frac{B_{\infty}^2}{64(e^{\varepsilon}-1)^2 \Delta} \\
\geq & (m-1)\frac{B_{\infty}^2}{128\varepsilon^2 \Delta} 
\end{align*}

 The second inequality is due to $D\left(p \| q\right) \leq \frac{(p-q)^2}{q(1-q)}$ and $\Delta/(B_{\infty}) \leq 0.35 \leq \frac{\sqrt{2}}{4}$. The last inequality is for the case that $0 < \varepsilon \leq 1/2$.

This special "MAB" problem can reduce to the stochastic CSB problem with $B_{\infty}$-bounded smoothness. We prove the lower bound by reduction.
\end{proof}

\subsection{Proof of Theorem \ref{theorem restate: lower bound for LDP}}
\label{section: proof of lower bound for LDP, L1}
\setcounter{theorem}{5}
\begin{theorem}
\label{theorem restate: lower bound for LDP}
(Restate) For any $m$ and $K$ such that $m/K$ is an integer, and any $\Delta$ satisfying $0 < \Delta/(B_1K) < 0.35$, the regret of any consistent $\varepsilon$-locally private algorithm $\pi$ on the CSB problem with $B_1$-bounded smoothness is bounded from below as  
$$
\liminf _{T \rightarrow \infty} \frac{Reg(T)}{\log T} \geq \frac{B_1^2(m-K) K}{64 (e^{\varepsilon}-1)^2 \Delta}
$$
Specifically, for $0 < \varepsilon \leq 1/2$, the regret is at least
$$
\liminf _{T \rightarrow \infty} \frac{Reg(T)}{\log T} \geq \frac{B_1^2(m-K)K}{128 \varepsilon^2 \Delta}
$$
\end{theorem}

Our lower bound is derived on the $K$-path semi-bandit problem~\citep{kveton2015tight}: There are $m$ base arms. The feasible super arms are $m/K$ paths. That is, path $i$ (super arm $i$) contains base arms $(i-1)K+1,...,iK$. Suppose the return of choosing super arm $S$ is $B_1$ times the sum of the weight $\hat{w}_i$ for $i \in S$. The weights of different base arms in the same super arm are identical, and the weights of base arms in different paths are distributed independently. Denote the best super arm as $S^*$. The weight of each base arm is a Bernoulli random variable with mean:

$$
\bar{w}(i)=\left\{\begin{array}{ll}{0.5} &
{ i \in S^* } \\ 
{0.5-\Delta / (B_1K)} & {\text { otherwise }}
\end{array}\right.
$$

To prove the lower bound, we adopt general canonical bandit model~\citep{lattimore2018bandit}. Denote the privacy-preserving algorithm as $\pi$, which maps the observation history to the probability of choosing each super arm, and the CSB instance as $\nu$,. The interaction between the algorithm and the instance in a given horizon $T$ can be denoted as the observation history $\mathcal{H}_{T} \triangleq\left\{\left(S_{t}, \bm{Z}_{t}\right)\right\}_{t=1}^{T}$. An observed history $\mathcal{H}_{T}$ is a random variable sampled from the measurable space $\left(([m]^{k} \times \mathbb{R}^{k})^{T}, \mathcal{B}([m]^{k} \times \mathbb{R}^{k})^{T}\right)$ and a probability measure $\mathbb{P}_{\pi \nu}$. $\mathbb{P}_{\pi \nu}$ is defined as follow:

\begin{itemize}
    \item The probability of choosing a super arm $S_t = S$ in step $t$ is dictated only by the algorithm $\pi(S|\mathcal{H}_{t-1})$.
    \item The distribution of rewards $\bm{X}_t$ in step $t$ is $f^{\nu}_{S_t}$, which depends on $S_t$ and conditionally independent on the history $\mathcal{H}_{t-1}$.
    \item In the case of local differential privacy, the algorithm cannot observe $X_t$ directly, but a privated version of rewards $\bm{Z}_t$. $\bm{Z}_t$ only depends on $X_t$ and is conditionally independent on the history $\mathcal{H}_{t-1}$. Denote the conditional distribution of $\bm{Z}$ as $M(\bm{Z}|\bm{X})$.
\end{itemize}

As a result, the distribution of the observed history $\mathcal{H}_T$ is
$$\mathbb{P}^T_{\pi \nu}\left(\mathcal{H}_{T}\right)=\prod_{t=1}^{T} \pi\left(S_{t} | \mathcal{H}_{t-1}\right) f^{\nu}_{S_{t}}\left(\bm{X}_{t}\right) M\left(\bm{Z}_t|\bm{X}_t\right).$$

Denote $g^{\nu}_{S_{t}}(\bm{Z}) = f^{\nu}_{S_{t}}\left(\bm{X}_{t}\right) M\left(\bm{Z}_t|\bm{X}_t\right)$. Before proving Theorem \ref{theorem restate: lower bound for LDP}, we state following two lemmas.

\begin{lemma}  
\label{lemma: decomposition lemma for LDP}
Given a stochastic CSB algorithm $\pi$ and two CSB environment $\nu_1$ and $\nu_2$, the KL divergence of two probability measure $\mathbb{P}^T_{\pi \nu_1}$ and $\mathbb{P}^T_{\pi \nu_2}$ can be decomposed as:
\begin{align*}
D\left(\mathbb{P}_{\pi \nu_{1}}^{T} \| \mathbb{P}_{\pi \nu_{2}}^{T}\right)= 
\sum_{t=1}^{T} \mathbb{E}_{\pi \nu_{1}}\left[D\left(\pi\left(S_{t} | \mathcal{H}_{t-1}, \nu_{1}\right) \| \pi\left(S_{t} | \mathcal{H}_{t-1}, \nu_{2}\right)\right)\right]+ 
\sum_{S \in \mathcal{S}} \mathbb{E}_{\pi \nu_{1}}\left[N_{S}(T)\right] D\left(g^{\nu_1}_{S} \| g^{\nu_2}_{S}\right),    
\end{align*}

$N_S(T)$ denotes the number of times $S$ is chosen in $T$ steps.
\end{lemma}

\begin{proof}
\begin{align*}
    D\left(\mathbb{P}_{\pi \nu_{1}}^{T} \| \mathbb{P}_{\pi \nu_{2}}^{T}\right) &=  \int_{\mathcal{H}_{T}} \ln \frac{\mathrm{d} \mathbb{P}_{\pi \nu_{1}}^{T}(H)}{\mathrm{d} \mathbb{P}_{\pi \nu_{2}}^{T}(H)} \mathrm{d}  \mathbb{P}_{\pi \nu_{1}}^{T}(H) \\
    &= \int_{\mathcal{H}_{T}} \sum_{t=1}^{T} \ln \frac{\pi(S_t|\mathcal{H}_{t-1},\nu_1)}{\pi(S_t|\mathcal{H}_{t-1},\nu_2)} \mathrm{d}  \pi(S_t|\mathcal{H}_{t-1},\nu_1) + \int_{\mathcal{H}_{T}} \sum_{t=1}^{T} \ln \frac{g^{\nu_1}_{S_{t}}(\bm{Z})}{g^{\nu_2}_{S_{t}}(\bm{Z})} \mathrm{d}  \left(g^{\nu_1}_{S_{t}}(\bm{Z})\right) \\
    &= \sum_{t=1}^{T} \mathbb{E}_{\pi \nu_{1}}\left[D\left(\pi\left(S_{t} | \mathcal{H}_{t-1}, \nu_{1}\right) \| \pi\left(S_{t} | \mathcal{H}_{t-1}, \nu_{2}\right)\right)\right] + \sum_{S \in \mathcal{S} }\left[\sum_{t=1}^{T} \mathbb{E}_{\mathbb{P}_{\pi}^{T} \nu_{1}}\left[\mathbbm{1}_{S_{t}=S}\right] D\left(g_{S}^{\nu_1}(\bm{Z}) \| g_{S}^{\nu_2}(\bm{Z})\right)\right] \\
    &= \sum_{t=1}^{T} \mathbb{E}_{\pi \nu_{1}}\left[D\left(\pi\left(S_{t} | \mathcal{H}_{t-1}, \nu_{1}\right) \| \pi\left(S_{t} | \mathcal{H}_{t-1}, \nu_{2}\right)\right)\right]+ 
\sum_{S \in \mathcal{S}} \mathbb{E}_{\pi \nu_{1}}\left[N_{S}(T)\right] D\left(g^{\nu_1}_{S} \| g^{\nu_2}_{S}\right)
\end{align*}

\end{proof}

\begin{lemma}
\label{lemma: upper bounds for KL divergence}
[Theorem 1 in \citet{duchi2016minimax}] For any $\alpha \geq 0$, let $Q$ be a conditional distribution that guarantees $\alpha$-differential privacy. Then for any pair of distributions $P_1$ and $P_2$, the induced marginal $M_1$ and $M_2$ satisfy the bound 
\begin{align*}
  D_{\mathrm{kl}}\left(M_{1} \| M_{2}\right)+D_{\mathrm{kl}}\left(M_{2} \| M_{1}\right) \leq 
  \min \left\{4, e^{2 \alpha}\right\}\left(e^{\alpha}-1\right)^{2}\left\|P_{1}-P_{2}\right\|_{\mathrm{TV}}^{2}.  
\end{align*}
\end{lemma}

~\\

Based on these two lemmas, we are now ready to prove Theorem \ref{theorem restate: lower bound for LDP}.

\begin{proof} (Proof of Theorem \ref{theorem restate: lower bound for LDP})
Suppose $\nu_1$ denote the stochastic CSB instance with weight vector:

$$
w(i)=\left\{\begin{array}{ll}{0.5} &
{ i \in S^* } \\ 
{0.5-\Delta / (B_1K)} & {\text { otherwise }}
\end{array}\right.
$$

For any sub-optimal super arm $S^{1}$, denote the CSB instance with the following weight vector as $\nu_2$:
$$
w(i)=\left\{\begin{array}{ll}{0.5} &
{ i \in S^* } \\ 
{0.5+\Delta/(B_1K)} & { i \in S^{1} } \\ 
{0.5-\Delta / (B_1K)} & {\text { otherwise }}
\end{array}\right.
$$

Denote the expected cumulative regret for a policy $\pi$ on instance $\nu$  in $T$ steps as $Reg(\pi,\nu, T)$. Then we have,
$$\operatorname{Reg}\left(\pi, \nu_{1}, T\right) \geq \mathbb{P}_{\pi \nu_{1}}\left(N_{S^{1}}(T) \geq T / 2\right) \frac{T \Delta}{2},$$
$$\operatorname{Reg}\left(\pi, \nu_{2}, T\right) \geq \mathbb{P}_{\pi \nu_{2}}\left(N_{S^1}(T) \leq T / 2\right) \frac{T \Delta}{2}$$

Combining these two inequality, we have
\begin{align}
\label{equation: pinsker for LDP}
\operatorname{Reg}\left(\pi, \nu_{1}, T\right) + \operatorname{Reg}\left(\pi, \nu_{2}, T\right)  
&\geq  \frac{ T\Delta}{2} \left(\mathbb{P}_{\pi \nu_{1}}\left(N_{S^1}(T) \leq T / 2\right)+\mathbb{P}_{\pi \nu_{2}}\left(N_{S^1}(T) \geq T / 2\right)\right) \notag \\
& \geq \frac{ T\Delta}{4} \exp \left(-D\left(\mathbb{P}^T_{\pi \nu_{1}} \| \mathbb{P}^T_{\pi \nu_{2}}\right)\right) 
\end{align}

The second inequality is due to probabilistic Pinsker's inequality~\citep{lattimore2019information}. 

By lemma \ref{lemma: decomposition lemma for LDP}, we have
\begin{align}
\label{equation: decomposition for LDP}
D\left(\mathbb{P}_{\pi \nu_{1}}^{T} \| \mathbb{P}_{\pi \nu_{2}}^{T}\right) 
=&\sum_{t=1}^{T} \mathbb{E}_{\pi \nu_{1}}\left[D\left(\pi\left(S_{t} | \mathcal{H}_{t}, \nu_{1}\right) \| \pi\left(S_{t} | \mathcal{H}_{t}, \nu_{2}\right)\right)\right]+ 
\sum_{S \in \mathcal{S}} \mathbb{E}_{\pi \nu_{1}}\left[N_{S}(T)\right] D\left(g^{\nu_1}_{S} \| g^{\nu_2}_{S}\right)  \notag \\
=&\sum_{S \in \mathcal{S}} \mathbb{E}_{\pi \nu_{1}}\left[N_{S}(T)\right] D\left(g^{\nu_1}_{S} \| g^{\nu_2}_{S}\right) \notag \\
=& \mathbb{E}_{\pi \nu_{1}}\left[N_{S^1}(T)\right] D\left(g^{\nu_1}_{S^1} \| g^{\nu_2}_{S^1}\right)
\end{align}
The second equality is because  $\pi$ chooses $S_t$ based on the observed history $\mathcal{H}_t$. The third equality is because $\nu_1$ and $\nu_2$ only differs in $S^1$.

By combining Equ. \ref{equation: pinsker for LDP} and Equ. \ref{equation: decomposition for LDP} we get,
\begin{align*}
    \label{equation: N_S for LDP}
    \mathbb{E}_{\pi \nu_{1}}\left[N_{S^1}(T)\right]   
    = & D\left(\mathbb{P}_{\pi \nu_{1}}^{T} \| \mathbb{P}_{\pi \nu_{2}}^{T}\right) / D\left(g^{\nu_1}_{S^1} \| g^{\nu_2}_{S^1}\right)  \\
    \geq & \ln (\frac{T\Delta}{4\left(\operatorname{Reg}\left(\pi, \nu_{1}, T\right) + \operatorname{Reg}\left(\pi, \nu_{2}, T\right)\right)}) / D\left(g^{\nu_1}_{S^1} \| g^{\nu_2}_{S^1}\right)\\
    \geq & \frac{\ln(T)/4-\ln(8m/K)}{ D\left(g^{\nu_1}_{S^1} \| g^{\nu_2}_{S^1}\right)} \\
    \geq & \frac{\ln(T)/4-\ln(8m/K)}{ \min \left\{4, e^{2 \varepsilon}\right\}\left(e^{\varepsilon}-1\right)^{2}\left\|f_{S^1}^{\nu_1}-f_{S^1}^{\nu_2}\right\|_{\mathrm{TV}}^{2}} \\
    \geq & \frac{\ln(T)/2-2\ln(8m/K)}{ \min \left\{4, e^{2 \varepsilon}\right\}\left(e^{\varepsilon}-1\right)^{2} D\left(f^{\nu_1}_{S^1} \| f^{\nu_2}_{S^1}\right)} \\
    %\geq & \frac{\ln(T)/2-2\ln(8m/K)}{ \min \left\{4, e^{2 \varepsilon}\right\}\left(e^{\varepsilon}-1\right)^{2} D\left(0.5 \| 0.5 + \Delta / (B_1K)\right)} \\
    \geq & \frac{K^2B_1^2\left(\ln(T)/16-\ln(8m/K)/8\right)}{ \min \left\{4, e^{2 \varepsilon}\right\}\left(e^{\varepsilon}-1\right)^{2} \Delta^2}
\end{align*}
The first inequality is due to Equ. \ref{equation: pinsker for LDP}. The second inequality is due to the consistent algorithm setting, i.e. $\operatorname{Reg}\left(\pi, \nu_{1}, T\right) \leq \frac{m}{k} \Delta T^p$. Here we set $p=3/4$. The third inequality is due to Lemma \ref{lemma: upper bounds for KL divergence}. The forth inequality is due to Pinsker's inequality. The last inequality is due to $D\left(p \| q\right) \leq \frac{(p-q)^2}{q(1-q)}$ and $\Delta/(B_1K) \leq 0.35 \leq \frac{\sqrt{2}}{4}$. 

%The last inequality is due to $D\left(p \| q\right) \leq \frac{(p-q)^2}{q(1-q)}$ and $\Delta/(B_1K) \leq 0.35 \leq \frac{\sqrt{2}}{4}$.

Now we can bound $
\liminf _{T \rightarrow \infty} \frac{Reg(T)}{\log T}
$:
\begin{align*}
 \liminf _{T \rightarrow \infty} \frac{Reg(T)}{\ln T} 
= & \liminf _{T \rightarrow \infty} \frac{\sum_{S \in \mathcal{S}, S \neq S^*} \Delta \cdot \mathbb{E}_{\pi \nu_{1}}\left[N_{S}(T)\right]}{\ln T} \\
\geq & \liminf _{T \rightarrow \infty} \frac{B_1^2\left(m/K-1\right) \Delta K^2\left(\ln(T)/16-\ln(8m/K)/8\right)}{\min \left\{4, e^{2 \varepsilon}\right\}\left(e^{\varepsilon}-1\right)^{2} \Delta^2 \ln T} \\
=& \frac{B_1^2mK}{16\min \left\{4, e^{2 \varepsilon}\right\}\left(e^{\varepsilon}-1\right)^{2} \Delta} \\
\geq& \frac{B_1^2mK}{128\varepsilon^{2} \Delta}
\end{align*}

The last inequality is due to $\left(e^{\varepsilon}-1\right)^{2} \leq 2 \varepsilon^2$ for $0 < \varepsilon \leq 1/2$.
\end{proof}

\subsection{Omitted Proof of Theorem \ref{theorem: utility for DP CMAB}}
\label{section: proof of upper bound for DP}

Before proving Theorem \ref{theorem: utility for DP CMAB}, we consider following two events, and show that these events happen with high probability.

\begin{lemma}
\label{lemma: high prob envent 1 for DP CMAB}
Let $\operatorname{Sum}_{t,i}$ be the sum of previous outcome $X_{t,i}$ without privacy noise for base arm $i$ in the first $t$ steps. We denote the following event as $\Lambda_{1}$: For any step $t \in [T]$ and any base arm $i \in [m]$, $$\left|\frac{\operatorname{Sum}_{t,i}}{T_{t,i}}-\mu_{i}\right| \leq \sqrt{\frac{4 \ln T}{T_{t,i}}}$$
Then $\Pr[\Lambda_1] \geq 1- 2/T$.
\end{lemma}

\begin{proof} 
The result follows directly from Hoeffding’s inequality and union bounds for all steps $t \in [T]$.
\end{proof}

\begin{lemma}
\label{lemma: high prob envent 2 for DP CMAB}
Let $\operatorname{Noise}_{t,i}$ be the Laplace noise added to $X_{t,i}$ in step $t$. We denote the following event as $\Lambda_{2}$: For any step $t \in [T]$ and any base arm $i \in [m]$, $$\left|\frac{\operatorname{Noise}_{t,i}}{T_{t,i}}\right| \leq \frac{ 12K \ln^3 T }{T_{t,i}\varepsilon}$$
Then $\Pr[\Lambda_2] \geq 1- 1/(mT)$.
\end{lemma}

\begin{proof} 
From the argument of our algorithm, $\operatorname{Noise}_{t,i}$ is the sum of at most $\log T$ i.i.d random variables drawn from $\operatorname{Lap}(2K\log T / \varepsilon)$. By the tail probability of Laplace distribution, we know that for any $\nu \sim \operatorname{Lap}(2K\log T / \varepsilon
)$, with prob. $1-\delta$, $|\nu| \leq 2K\log T  \ln (1/\delta) / \varepsilon$. Set $\delta = 1/(m^2T^2 \log T)$. By union bounds over $\log T$ random variables, we have $|\operatorname{Noise}_{t,i}| \leq 4K\log^2 T \ln (mT \log T) / \varepsilon $ with prob. $1-1/(m^2T^2)$ for a fixed $i$ and $t$. By union bound over all base arm $i$ and step $t$, we prove that

$$\left|\frac{\operatorname{Noise}_{t,i}}{T_{t,i}}\right| \leq \frac{ 4K \log^2 T \ln(mT\log T)}{T_{t,i}\varepsilon}\leq \frac{ 12K \ln^3 T }{T_{t,i}\varepsilon}$$
 for any step $t$ and base arm $i$ for sufficiently large T with prob. $1-1/(mT)$.
\end{proof}

\begin{proof} (Proof of Lemma \ref{lemma: regret decomposition for DP})   
Suppose $G_t$ denote the event that the oracle fails to produce an $\alpha$-approximate answer with respect to the input vector in step $t$. Similar with the proof of Theorem \ref{theorem restate: utility LDP CUCB2}, the cumulative regret in the steps that $G_t$ happens is at most $R_{\operatorname{fail}}\leq (1-\beta)T\Delta_{\max}$.

Then we have,
\begin{align}
    Reg_{\mu,\alpha, \beta}(T) 
    \leq &T \alpha \beta \operatorname{opt}_{\mu} - \mathbb{E}\sum_{t=1}^T r_{\mu}(S_t)\notag\\
    \leq & R_{\operatorname{fail}} + T \alpha \beta \operatorname{opt}_{\mu} - 
    \left(T\alpha \operatorname{opt}_{\mu} -\sum_{t\in [T]}\Delta_{t} \mathbbm{1}\{\neg G_{t}\}\right)\notag \\
    \leq &\sum_{t\in [T]}\Delta_{t} \mathbbm{1}\{\neg G_{t}\}\notag
\end{align}
Here $\Delta_{t}$ denote the sub-optimal gap in step $t$.

This means that we only need to consider the steps that $G_t$ doesn't happen.
Denote $\hat{R}(T)$ as the regret if event $\Lambda_1$ and $\Lambda_2$ happen. 
\begin{align*}
    Reg_{\mu,\alpha,\beta}(T) \leq & \Pr\{\Lambda_1 \cap \Lambda_2\} \hat{R}(T) + \sum_{i\in[m]} \Delta_{\operatorname{min}}^{i} \\
    &+ \Pr\{\neg \Lambda_1\} T \Delta_{\operatorname{max}} + \Pr\{\neg \Lambda_1\} T \Delta_{\operatorname{max}} \\
    \leq & \hat{R}(T) + (m+2)\Delta_{\operatorname{max}} 
\end{align*}
If event $\Lambda_1$ and $\Lambda_2$ happen, we have 
\begin{align*}
    \left|\tilde{\mu}_{t}(i) - \mu_i\right| & =  \left| \frac{\operatorname{Sum}_{t,i}}{T_{t,i}} - \mu_i + \frac{\operatorname{Noise}_{t,i}}{T_{t,i}}\right| \\
    &\leq  \sqrt{\frac{4 \ln T}{T_{t,i}}} +  \frac{12 K \ln^3 T}{T_{t,i}}
\end{align*}

for step $t \in [T]$, if we choose a sub-optimal super arm with sub-optimality gap $\Delta_{S_t} > 0$, then we have 

\begin{align}
    \alpha r_{\mu}(S_{\mu}^*)- r_{\mu}(S_t) \leq & \alpha r_{\bar{\mu}_t}(S_{\mu}^*) - (r_{\bar{\mu}_t}(S_t) - B_{1}\|\bar{\mu}_{t}-\mu\|_{1}) \notag\\
    \leq & B_{1}\|\bar{\mu}_{t}-\mu\|_{1} \notag\\
    \leq & B_{1}(\|\bar{\mu}_{t}-\tilde{\mu}_t\|_{1}+\|\tilde{\mu}_{t}-\mu\|_{1}) \notag\\
    \leq & B_{1} \sum_{i \in S_t} \left(4\sqrt{\frac{\ln T}{T_{t-1,i}}} + \frac{24 K\ln^{3} T}{T_{t-1,i} \varepsilon} \right) 
\end{align}
The first inequality is due to $L_{1}$ smoothness assumption. The second inequality is because the oracle returns $S_t$ which satisfies $r_{\bar{\mu}_t}(S_t) \geq \alpha r_{\bar{\mu}_t}(S_{\mu}^*)$. The last inequality is due to the definition of $\bar{\mu}_t$ and the concentration bound for $\tilde{\mu}_t$.

This shows that if event $\Lambda_1$ and $\Lambda_2$ happen, and we choose a sub-optimal super arm with sub-optimality gap $\Delta_{S_t} > 0$ in step t, $F_t$ happens.

Then we have $\hat{R}(T) \leq \sum_{t\in[T]}\Delta_{S_t} \mathbf{1}\{F_t\}$, which finishes the proof.
\end{proof}

\subsection{Proof of Theorem \ref{theorem restate: lower bound for DP}}

\label{section: proof of lower bound for DP}

\setcounter{theorem}{8}
\begin{theorem} 
\label{theorem restate: lower bound for DP}
For any $m$ and $K$ such that $m \geq 2K$,and any $\Delta$ satisfying $0 < \Delta/(B_1K) < 0.35$, the regret for any consistent $\varepsilon$-DP algorithm on the CSB problem with $B_1$ bounded smoothness is at least $\Omega\left(\frac{B_1^2mK \ln T}{\Delta} + \frac{B_1m K\ln T}{\varepsilon}\right)$.
\end{theorem}

\begin{proof} 
Previous results have shown that the regret for any non-private CSB algorithm is at least $\Omega\left(\frac{mK \ln T}{\Delta}\right)$~\citep{kveton2015tight}. They consider linear CSB problem, which is a special case of $B_1$ bounded smoothness CSB with $B_1 = 1$. We slightly modify the hard instance in \citet{kveton2015tight} and prove the regret lower bound for $B_1$ bounded smoothness CSB in non-private setting.

The main difference is that we assume the reward of any super arms $S_t$ is $B_1$ times the sum of weights $w(i)$ for $i \in S_t$. In our hard instance, we also consider the $K$-path semi-bandit problem. There are $m$ base arms. The feasible super arms are $m/K$ paths. Path $i$ (Super arm $i$) contains base arms $(i-1)K+1,(i-1)K+2,...,iK$. The weight of  base arm $i$ is a Bernoulli random variable with mean $\bar{w}(i)$. Since $\Delta$ in our setting is $B_1$ times that of the instance in \citet{kveton2015tight}, we slightly modify the mean of $w(i)$ to make sure that the mean $\bar{w}(i) \in [0,1]$:

$$
\bar{w}(i)=\left\{\begin{array}{ll}{0.5} &
{ i \in S^* } \\ 
{0.5-\Delta / (B_1K)} & {\text { otherwise }}
\end{array}\right.
$$

With the same argument in \citet{kveton2015tight}, we can prove that each path need to be selected at least $\frac{B_1^2K^2\ln T}{\Delta^2}$ times. which means that the regret is at least $\frac{B_1^2K^2\ln T}{\Delta^2}\Delta \cdot (L/K-1) = \Omega\left(\frac{B_1^2mK\ln T}{\Delta}\right)$. Since private CSB is harder than non-private CSB (There is a reduction from non-private CSB to private CSB), the regret of private CSB is at least $\Omega\left(\frac{B_1^2mK\ln T}{\Delta}\right)$.

By the following lemma, we can show that the regret of any $\varepsilon$-DP consistent CSB algorithm is at least $\Omega\left(\frac{B_1mK \ln T}{\varepsilon}\right)$. Combining both results, we can prove that the regret lower bound is $\Omega\left( \max\left\{\frac{B_1^2mK \ln T}{\Delta}, \frac{B_1mK \ln T}{\varepsilon}\right\} \right) = \Omega\left(\frac{B_1^2mK \ln T}{\Delta} + \frac{B_1m K\ln T}{\varepsilon}\right)$.
\end{proof}

\begin{lemma}
\label{lemma: lower bound for DP}
For any $m$ and $K$ such that $m \geq 2K$, and any $\Delta$ satisfying $0 < \Delta/(B_1K) < 0.35$, the regret for any consistent CSB algorithm guaranteeing $\varepsilon$-DP is at least $\Omega\left(\frac{B_1m K\ln T}{\varepsilon}\right)$.
\end{lemma}

Now we only need to prove Lemma \ref{lemma: lower bound for DP}.
\begin{proof}
 We consider the  CSB instance: Suppose there are $m$ base arms, each associated with a weight sampled from Bernoulli distribution. These $m$ base arms are divided into three sets, $S^*, \tilde{S}, \bar{S}$. $S^*$ contains $m$ base arms, which build up the optimal super arm set. $\tilde{S}$ contains $K-1$ “public“ base arms for sub-optimal super arms. These arms are contained in all sub-optimal super arms. $\bar{S}$ contains $m-2K+1$ base arms. each base arm combined with $K-1$ "public" base arms in $\tilde{S}$ builds up a sub-optimal super arm. Totally we have $m-2K+1$ sub-optimal super arms and one optimal super arm. The mean of the Bernoulli random variable associated to each base arm is defined as follow:

$$
w(i)=\left\{\begin{array}{ll}{0.5} &
{ i \in S^* } \\ 
{0.5-\Delta / (B_1K)} & {\text { otherwise }}
\end{array}\right.
$$

The weights of base arms in $\tilde{S}$ are identical, while other weights are i.i.d sampled. The reward of pulling a super arm $S$ is $B_1$ times the sum of weights of all base arm $i \in S$. As a result, the sub-optimality gap of each sub-optimal super arm is $\Delta$. We denote this CSB instance as $\nu_1$.

Now we fix one certain sub-optimal super arm $S_1$. Denote $E_{S_1}$ as the event that super arm $S_1$ is pulled $\leq \frac{B_1K \ln T}{400\varepsilon\Delta}:= t_S$ times. Our goal is to show that $E_{S_1}$ happens with probability at most $\frac{1}{2m}$. If this is true, by union bounds over all sub-optimal super arms, all the sub-optimal super arms will be pulled at least $t_S$ times with prob. $1-\frac{1}{2}$. This means the regret is at least $\Omega\left(\frac{B_1mK \ln T}{\varepsilon} \right)$.

Now we prove that $P_{\nu_1}(E_{S_1}) \leq 1/(2m)$. Our analysis is inspired by the work of \citet{shariff2018differentially}. Consider another CSB instance with all the setting the same as $\nu_1$, except that the mean weights of base arms in $S_1$ are increased by $2\Delta/(B_1K)$ each. We denote this instance as $\nu_2$. Consider the case that rewards are drawn from $\nu_2$. Due to consistent property, the regret of the algorithm is at most $T^{3/4}m\Delta$. For sufficiently large $T$, we have 
$$\frac{T\Delta}{2K}\mathbb{P}_{\nu_2}[E] \leq \frac{(T-t_S)\Delta}{K}\mathbb{P}_{\nu_2}[E] \leq T^{3/4}m\Delta$$.

The first inequality is for sufficiently large $T$. The second inequality is because if $E$ happens in $\nu_2$, the regret is at least $(T-t_s)\cdot \frac{\Delta}{K}$. This means that $\mathbb{P}_{\nu_2}[E] \leq \frac{mK}{T^{1/4}}$.

Now we consider the influence of differential privacy. The result of \citet{karwa2017finite} (Lemma 6.1) states that the group privacy between the case that inputs are drawn i.i.d from distribution $P_1$ and $P_2$ is proportional to $6\varepsilon n \cdot d_{\mathrm{TV}}(P, Q)$, where n is the number of inputs data. We apply the coupling argument in \citet{karwa2017finite} to our setting. Suppose the algorithm turns to an oracle when she needs to sample a reward of super arm $S_1$. The oracle can generate at most $t_S$ pairs of data. The left ones are i.i.d sampled from $\nu_1$, while the right ones are i.i.d sampled from $\nu_2$. Whether the algorithm receive a reward sampled from the left or the right depends on the true environment. The algorithm turns to another oracle if and only if the original oracle runs out of $t_S$ samples. By Lemma 6.1 in \citet{karwa2017finite}, the oracle runs out of $t_S$ samples, i.e. event $E_{S_1}$ happens with similar probability under $\nu_1$ and $\nu_2$. Indeed, the probability of event $E_{S_1}$ happens under $\nu_1$ is less than $\exp{\left(6\varepsilon t_{S} \cdot d_{\mathrm{TV}}(P, Q)\right)}$  times the probability of event $E_{S_1}$ happens under $\nu_2$.

That is, for sufficiently large $T$,
\begin{align*}
    \mathbb{P}_{\nu_1}[E_{S_1}] \leq &\exp{\left(6\varepsilon t_{S} \cdot d_{\mathrm{TV}}(\nu_1, \nu_2)\right)}\mathbb{P}_{\nu_2}[E_{S_1}] \\
    \leq & \exp{\left(24\varepsilon t_{S} \cdot \frac{\Delta}{B_1K}\right)}\mathbb{P}_{\nu_2}[E_{S_1}]\\
    \leq & \exp{\left(0.06 \ln T\right)} \frac{mK}{T^{1/4}} \\
    = & mKT^{-0.19}
    \leq  \frac{1}{2m}. \\
\end{align*}
The second inequality is due to $d_{\mathrm{TV}}(\nu_1, \nu_2) \leq \sqrt{\frac{D_{KL}(\nu_1\|\nu_2)}{2}} \leq 4\Delta/(B_1K)$ by Pinsker's inequality and the setting that the public base arms are identical.

\end{proof}


\end{document}